\documentclass[final,12pt]{arxiv}
\usepackage{natbib}

\usepackage[utf8]{inputenc} 
\usepackage[T1]{fontenc}    
\usepackage{hyperref}       
\usepackage{url}            
\usepackage{booktabs}       
\usepackage{amsfonts}       
\usepackage{nicefrac}       
\usepackage{microtype}      
\usepackage{xcolor}         

\usepackage{dsfont}
\usepackage{enumitem}
\usepackage[toc,page,header]{appendix}

\usepackage[english]{babel}
\usepackage{amsmath, amssymb}
\usepackage{graphicx}
\usepackage{latexsym}
\usepackage{graphicx}
\usepackage{cases}
\usepackage[mathscr]{euscript}
\usepackage{xcolor}
\usepackage{colortbl}
\usepackage{thmtools}
\usepackage{thm-restate}

\usepackage{mathtools}
\usepackage{subcaption}

\usepackage{multirow}
\usepackage{wrapfig}

\usepackage{pifont}

\usepackage{MnSymbol}
\DeclareMathAlphabet\mathbb{U}{msb}{m}{n}
\usepackage{xpatch}

\definecolor{Gray}{gray}{0.85}
\newcolumntype{g}{>{\columncolor{Gray}}c}

\hypersetup{
  breaklinks   = true, 
  colorlinks   = true, 
  urlcolor     = blue, 
  linkcolor    = blue, 
  citecolor   = blue 
}

\def\Rset{\mathbb{R}}

\DeclareMathOperator*{\E}{\mathbb{E}}

\DeclareMathOperator*{\argmax}{\rm argmax}
\DeclareMathOperator*{\argmin}{\rm argmin}

\newcommand{\nrm}[1]{{\left\vert\kern-0.25ex\left\vert\kern-0.25ex\left\vert #1 
    \right\vert\kern-0.25ex\right\vert\kern-0.25ex\right\vert}}

\DeclarePairedDelimiter{\abs}{\lvert}{\rvert} 
\DeclarePairedDelimiter{\bracket}{[}{]}
\DeclarePairedDelimiter{\curl}{\{}{\}}
\DeclarePairedDelimiter{\paren}{(}{)}

\newcommand{\sA}{{\mathscr A}}

\newcommand{\sC}{{\mathscr C}}
\newcommand{\sD}{{\mathscr D}}
\newcommand{\sE}{{\mathscr E}}
\newcommand{\sF}{{\mathscr F}}

\newcommand{\sH}{{\mathscr H}}

\newcommand{\sM}{{\mathscr M}}

\newcommand{\sT}{{\mathscr T}}

\newcommand{\sX}{{\mathscr X}}
\newcommand{\sY}{{\mathscr Y}}

\newcommand{\sfL}{{\mathsf L}}

\newcommand{\h}{\widehat}
\newcommand{\ov}{\overline}
\newcommand{\uv}{\underline}

\newcommand{\e}{\epsilon}

\newcommand{\ignore}[1]{}

\newcommand{\Rad}{\mathfrak R}

\newcommand{\hh}{{\sf h}}

\newcommand{\compsum}{{cross-entropy}}

\newcommand{\1}{\mathds{1}}
\newcommand{\labs}{{\sfL_{\rm{abs}}}}
\newcommand{\labsc}{{\sfL_{\rm{abs}}}}
\newcommand{\lsc}{{\sfL}}

\usepackage{times}
\title[Theoretically Grounded Score-Based Multi-Class Abstention]{Theoretically Grounded Loss Functions and Algorithms for Score-Based Multi-Class Abstention}

\arxivauthor{%
 \Name{Anqi Mao} \Email{aqmao@cims.nyu.edu}\\
 \addr Courant Institute of Mathematical Sciences, New York%
 \AND
 \Name{Mehryar Mohri} \Email{mohri@google.com}\\
 \addr Google Research and Courant Institute of Mathematical Sciences, New York%
 \AND
 \Name{Yutao Zhong} \Email{yutao@cims.nyu.edu}\\
 \addr Courant Institute of Mathematical Sciences, New York%
}

\begin{document}

\maketitle

\begin{abstract}
\ignore{
We analyze the important framework of learning with abstention in the
multi-class classification setting, where the learner has the option
of abstaining at some cost. We present a series of new theoretical and
algorithmic results for the score-based formulation. We derive and
introduce new families of surrogate losses, which include the
state-of-the-art surrogate losses as special cases in the single-stage
setting and a novel family of loss functions in the two-stage
setting. We prove $\sH$-consistency bounds for these surrogate losses,
which are strong non-asymptotic and hypothesis set-specific guarantees
upper-bounding the estimation error of the abstention loss function in
terms of the estimation error of the surrogate loss. Our bounds can
guide the design of novel abstention algorithms by minimizing the
proposed surrogate losses and by enabling the comparison of different
\compsum\ score-based surrogates, which is aligned with our empirical
findings. We carried out extensive experiments on CIFAR-10, CIFAR-100,
and SVHN datasets to compare the performance of these new algorithms
with state-of-the-art ones. The results demonstrate the practical
significance of our new surrogate losses and two-stage abstention
algorithms. They also show that the relative performance of the
state-of-the-art \compsum\ score-based surrogate losses can vary
across datasets.}
Learning with abstention is a key scenario where the learner can
abstain from making a prediction at some cost.  In this paper, we
analyze the score-based formulation of learning with abstention in the
multi-class classification setting.
We introduce new families of surrogate losses for the abstention loss
function, which include the state-of-the-art surrogate losses in the
single-stage setting and a novel family of loss functions in the
two-stage setting. We prove strong non-asymptotic and hypothesis
set-specific consistency guarantees for these surrogate losses, which
upper-bound the estimation error of the abstention loss function in
terms of the estimation error of the surrogate loss.
Our bounds can help compare different score-based surrogates and guide
the design of novel abstention algorithms by minimizing the proposed
surrogate losses. We experimentally evaluate our new algorithms on
CIFAR-10, CIFAR-100, and SVHN datasets and the practical significance
of our new surrogate losses and two-stage abstention algorithms. Our
results also show that the relative performance of the
state-of-the-art score-based surrogate losses can vary across
datasets.
\end{abstract}



\section{Introduction}
\label{sec:intro}

In many applications, incorrect predictions can be costly and it is
then preferable to abstain from making predictions for some input
instances, since the cost of abstention is typically less significant.
As an example, in medical diagnosis, the cost of an incorrect
diagnosis is incommensurable since the patient's health may be
jeopardized. In contrast, the cost of abstention is typically that of
resorting to some additional laboratory tests.  For a spoken-dialog
system, an incorrect prediction may result in canceling a credit card,
for a bank, or shipping the wrong medication to the patient, for a
pharmacy, while the cost of abstaining is limited to that of switching
to a human operator.

A related problem arises in applications where a learning model
distilled from a very complex one is used, due to its more modest
inference cost.  However, since it is less accurate, one may need to
resort to abstention for some inputs and instead predict using the
more complex model, despite its higher inference cost. This problem of
\emph{deferring} to an alternative model, in fact to a human in some
cases, can also be viewed as a special case of the general abstention
scenario \citep{madras2018predict,raghu2019algorithmic,
  mozannar2020consistent,
  okati2021differentiable,wilder2021learning,verma2022calibrated,
  narasimhanpost,verma2023learning}.  In other applications such as
information extraction or natural language text generation or
question-answering, the output is sometimes not factual
\citep{Filippova2020, maynez2020}. It can then be important to learn
to abstain from responding to avoid such \emph{hallucinations} and
instead defer to a more costly predictor.

The scenario of classification with abstention is very broad and
admits increasingly many important applications, including as a
subroutine for other algorithms such as active learning
\citep{ZhangChaudhuri2016} or dual purpose learning
\citep{AminDeSalvoRostamizadeh2021}. But, how should we formulate the
problem of multi-class classification with abstention and when should
we abstain?

There is a vast literature related to the problem of abstention or
rejection. Here, we briefly discuss work directly related to this
study and give a more detailed discussion in
Appendix~\ref{app:related-work}.  A standard method for abstention
adopted in the past, which covers a very large number of publications
(e.g., \citet{HerbeiWegkamp2005,bartlett2008classification,
  yuan2010classification,lei2014classification,denis2020consistency})
and dates back to the early work of \citet{Chow1957,chow1970optimum},
is the so-called \emph{confidence-based abstention}. This consists of
first learning a predictor and then abstaining when the score returned
by the predictor falls below some fixed
threshold. \citet{HerbeiWegkamp2005} examined binary classification
with abstention by giving the optimal rule for these ternary
functions. \citet{bartlett2008classification} formulated a loss
function for this setting taking into consideration the abstention
cost $c$ and suggested to learn a predictor using a \emph{double hinge
loss} that they showed benefits from consistency
results. \citet{yuan2010classification} investigated the necessary and
sufficient condition for consistency of convex risk minimization with
respect to the abstention loss and obtained the corresponding excess
error bounds in the same setting. Other variants of this framework
have also been studied in \citep{lei2014classification,
  denis2020consistency}.

However, \citet*{CortesDeSalvoMohri2016,CortesDeSalvoMohri2023} argued
that, in general, confidence-based abstention is suboptimal, unless
the predictor learned is the Bayes classifier. They showed that, in
general, even in simple cases, no threshold-based abstention can
achieve the desired result. They introduced a novel framework for
abstention that consists of learning \emph{simultaneously} both a
predictor $h$ and a rejector $r$ that, in general, can be distinct
from a threshold-based function. They further defined a
\emph{predictor-rejector formulation} loss function for the pair $(h,
r)$, taking into consideration the abstention cost $c$. The authors
gave Rademacher complexity-based generalization bounds for this
learning problem. They also suggested several surrogate loss functions
for the abstention loss in the binary classification setting, and
further showed that these surrogate losses benefitted from consistency
guarantees. They designed algorithms based on these surrogate losses,
which they showed empirically outperform confidence-based abstention
baselines.  This work had multiple follow-up studies, including a
theoretical and algorithmic study of boosting with abstention
\citep{CortesDeSalvoMohri2016bis} and a study of the extension of the
results to multi-class setting \citep{NiCHS19}. These authors argued
that the design of calibrated or Bayes-consistent surrogate losses in
the multi-class classification setting based on the predictor-rejector
abstention loss of \citet{CortesDeSalvoMohri2016} was difficult and
left that as an open problem. Recently,
\citet{MaoMohriZhong2024predictor} introduced several new theoretical
and algorithmic findings within this framework, effectively addressing
the open question. Furthermore,
\citet{MohriAndorChoiCollinsMaoZhong2024learning} explored the
framework from the perspective of learning with a fixed predictor,
applying their novel algorithms to decontextualization tasks.

\citet{mozannar2020consistent} proposed instead for the multi-class
abstention setting a \emph{score-based formulation}, where, in
addition to the standard scoring functions associated to each label, a
new scoring function is associated to a new rejection label.
Rejection takes places when the score given to the rejection label is
higher than other scores and the rejector is therefore implicitly
defined via this specific rule. The authors suggested a surrogate loss
for their approach based on the cross-entropy (logistic loss with
softmax applied to neural networks outputs), which they proved to be
Bayes-consistent. More recently, \citet{caogeneralizing} gave a more
general family of Bayes-consistent surrogate losses for the
score-based formulation that can be built upon any consistent loss for
the standard multi-class classification problem. Most recent research
by \citet{pmlr-v206-mozannar23a} demonstrates that cross-entropy
score-based surrogate losses are not realizable $\sH$-consistent \citep{long2013consistency,zhang2020bayes}, in relation to
abstention loss. Instead, the authors propose a novel surrogate
loss that is proved to be realizable $\sH$-consistent when $\sH$ is
\emph{closed under scaling}, although its Bayes-consistency remains
unclear. The challenge of devising a surrogate loss that exhibits both
Bayes-consistency and realizable $\sH$-consistency remains an open
problem.

This paper presents a series of new theoretical and algorithmic
results for multi-class classification for the score-based abstention
formulation. In Section~\ref{sec:preliminary}, we formalize the
setting and first define explicitly the underlying abstention loss. We
then show how the general family of surrogate losses introduced by
\citet{caogeneralizing} can be naturally derived from that expression
in Section~\ref{sec:score-general}.

More importantly, we prove \emph{$\sH$-consistency bounds} for these
surrogate losses (Section~\ref{sec:score-bounds}), which are
non-asymptotic and hypothesis set-specific guarantees upper-bounding
the estimation error of the abstention loss function in terms of the
estimation error of the surrogate loss \citep{awasthi2022multi}. These
provide stronger guarantees than Bayes-consistency guarantees, which
only provide an asymptotic guarantee and hold only for the full family
of measurable functions. We first derive our guarantees for a broad
family of score-based abstention surrogates, which we name
\emph{\compsum\ score-based surrogate losses}. These include the
surrogate losses in \citep{mozannar2020consistent,caogeneralizing},
for which our guarantees admit their Bayes-consistency as a special
case. Our theory can also help compare different surrogate losses. To
make it more explicit, we give an explicit analysis of the
minimizability gaps appearing in our bounds. We further prove a
general result showing that an $\sH$-consistency bound in standard
classification yields immediately an $\sH$-consistency bound for
score-based abstention losses. Minimization of these new surrogate
losses directly leads to new algorithm for multi-class abstention.

In Section~\ref{sec:two-stage}, we analyze a two-stage algorithmic
scheme often more relevant in practice, for which we give surrogate
losses that we prove to benefit from $\sH$-consistency bounds. These
are also non-asymptotic and hypothesis set-specific guarantees
upper-bounding the estimation error of the abstention loss function in
terms of the estimation error of the first-stage surrogate loss and
second-stage one. Minimizing these new surrogate losses directly leads
to new algorithm for multi-class abstention.

In Section~\ref{sec:realizable}, we demonstrate that our proposed
two-stage score-based surrogate losses are not only Bayes-consistent,
but also realizable $\sH$-consistent. This effectively addresses the
open question posed by \citet{pmlr-v206-mozannar23a} and highlights
the benefits of the two-stage formulation.

In Section~\ref{sec:finite-sample}, we show that our $\sH$-consistency
bounds can be directly used to derive finite sample estimation bounds
for a surrogate loss minimizer of the abstention loss. These are more
favorable and more relevant guarantee than a similar finite sample
guarantee that could be derived from an excess error bound.

In Section~\ref{sec:experiments}, we report the results of several
experiments comparing these algorithms and discuss them in light of
our theoretical guarantees.  Our empirical results show, in
particular, that the two-stage score-based abstention surrogate loss
consistently outperforms the state-of-the-art cross-entropy
scored-based abstention surrogate losses on CIFAR-10, CIFAR-100 and
SVHN, while highlighting that the relative performance of the
state-of-the-art \compsum\ scored-based abstention losses varies by
the datasets. We present a summary of our main contribution as follows
and start with a formal description of the problem formulations.
\begin{itemize}

    \item Derivation of the cross-entropy score-based surrogate loss
      from first principles, which include the state-of-the-art
      surrogate losses as special cases.

    \item $\sH$-consistency bounds for cross-entropy score-based
      surrogate losses, which can help theoretically compare different
      cross-entropy score-based surrogate losses and guide the design
      of a multi-class abstention algorithm in comparison to the
      existing asymptotic consistency guarantees.

    \item A novel family of surrogate loss functions in the two-stage
      setting and their strong $\sH$-consistency bounds guarantees.

    \item Realizable $\sH$-consistency guarantees of proposed
      two-stage score-based surrogate loss, which effectively
      addresses the open question posed by
      \citet{pmlr-v206-mozannar23a} and highlights the benefits of the
      two-stage formulation.

    \item Extensive experiments demonstrating the practical
      significance of our new surrogate losses and the varying
      relative performance of the state-of-the-art cross-entropy
      score-based surrogate losses across datasets.
\end{itemize}

\section{Preliminary}
\label{sec:preliminary}

We consider the standard multi-class classification setting with an
input space $\sX$ and a set of $n \geq 2$ classes or labels $\sY =
\curl*{1, \ldots, n}$. We will denote by $\sD$ a distribution over
$\sX \times \sY$ and by $p(x, y)$, the conditional probability of $Y =
y$ given $X = x$, that is $p(x, y) = \sD(Y = y \!\mid\! X = x)$. We
will also use $p(x) = \paren*{p(x, 1), \ldots, p(x, n)}$ to denote the
vectors of these probabilities for a given $x$.

We study the learning scenario of multi-class classification with
abstention in the
\emph{score-based formulation} proposed by
\citet{mozannar2020consistent} and recently studied by
\citet{caogeneralizing}.

\paragraph{Score-Based Abstention Formulation}

In this formulation of the abstention problem, the label set $\sY$ is
augmented with an additional category $(n + 1)$ corresponding to
abstention. We denote by $\sY \cup \curl*{n+1} = \curl*{1, \ldots, n,
  n + 1}$ the augmented set and consider a hypothesis set $ \sH$ of
functions mapping from $\sX \times (\sY \cup \curl*{n + 1})$ to
$\Rset$.
The label associated by $ h \in \sH$ to an input $x \in \sX$ is
denoted by $ \hh(x)$ and defined by $ \hh(x) = n + 1$ if $ h(x, n + 1)
\geq \max_{y \in \sY} h(x, y)$; otherwise, $ \hh(x)$ is defined as an
element in $\sY$ with the highest score, $ \hh(x) = \argmax_{y \in
  \sY} h(x, y)$, with an arbitrary but fixed deterministic strategy
for breaking ties.  When $ \hh(x) = n + 1$, the learner abstains from
making a prediction for $x$ and incurs a cost $c(x)$. Otherwise, it
predicts the label $y = \hh(x)$. The \emph{score-based abstention
loss} $\labsc$ for this formulation is defined as follows for any $ h
\in \sH$ and $(x, y) \in \sX \times \sY$:
\begin{equation}
\label{eq:abs-score}
\labsc( h, x, y)
= \1_{ \hh(x)\neq y}\1_{ \hh(x)\neq n + 1} + c(x) \1_{ \hh(x) = n + 1}.
\end{equation}
Thus, when it does not abstain, $ \hh(x) \neq n + 1$, the learner
incurs the familiar zero-one classification loss and when it abstains,
$ \hh(x) = n + 1$, the cost $c(x)$.  Given a finite sample drawn
i.i.d.\ from $\sD$, the learning problem consists of selecting a
hypothesis $ h$ in $ \sH$ with small expected score-based abstention
loss, $\E_{(x, y) \sim \sD}[\labsc( h, x, y)]$. Note that the cost $c$
implicitly controls the rejection rate when minimizing the abstention
loss.

Optimizing the score-based abstention loss is intractable for most
hypothesis sets. Thus, instead, learning algorithms for this scenario
must resort to a surrogate loss $\lsc$ for $\labsc$. In the next
sections, we will define score-based surrogate losses and analyze
their properties. Given a loss function $ \sfL$, we denote by
$\sE_{\lsc}( h) = \E_{(x, y) \sim \sD}\bracket*{\lsc( h, x, y)}$ the
generalization error or expected loss of $ h$ and by $\sE_{\lsc}^*(
\sH) = \inf_{ h \in \sH} \sE_{\lsc}( h)$ the minimal generalization
error. In the following, to simplify the presentation, we assume that
the cost function $c\in (0,1)$ is constant. However, many of our
results extend straightforwardly to the general case.

\paragraph{$\sH$-Consistency Bounds}

We will seek to derive \emph{$\sH$-consistency bounds} for
$\lsc$. These are strong guarantees that take the form of inequalities
establishing a relationship between the abstention loss $\labsc$ of
any hypothesis $h \in \sH$ and the surrogate loss $\lsc$ associated
with it
\citep{awasthi2021calibration,awasthi2021finer,awasthi2022Hconsistency,
  awasthi2022multi,AwasthiMaoMohriZhong2023theoretically,awasthi2024dc,
  MaoMohriZhong2023cross,MaoMohriZhong2023ranking,
  MaoMohriZhong2023rankingabs,zheng2023revisiting,
  MaoMohriZhong2023characterization,MaoMohriZhong2023structured,mao2024top,mao2024h}. These
are bounds of the form $\sE_{\labsc}( h) - \sE_{\labsc}^*( \sH) \leq
f\paren*{\sE_{\lsc}( h) - \sE_{\lsc}^*( \sH)}$, for some
non-decreasing function $f$, that upper-bound the estimation error of
the loss $\labsc$ in terms of that of $\lsc$ for a given hypothesis
set $ \sH$. Thus, they show that if we can reduce the surrogate
estimation error $(\sE_{\lsc}( h) - \sE_{\lsc}^*( \sH))$ to $\e > 0$,
then the estimation error of $\labsc$ is guaranteed to be at most
$f(\e)$. These guarantees are non-asymptotic and take into
consideration the specific hypothesis set $\sH$ used.

\paragraph{Minimizability Gaps} 

A key quantity appearing in these bounds
is the \emph{minimizability gap}, denoted by $\sM_{\lsc}(\sH)$ and
defined by $\sM_{\lsc}(\sH) = \sE^*_{\lsc}( \sH) - \E_x
\bracket[\big]{\inf_{ h \in \sH} \E_y\bracket*{\lsc( h, X, y) \mid X =
    x}}$ for a given hypothesis set $\sH$. Thus, the minimizability
gap for a hypothesis set $\sH$ and loss function $\lsc$ measures the
difference of the best-in-class expected loss and the expected
pointwise infimum of the loss.  Since the infimum is super-additive,
it follows that the minimizability gap is always non-negative.
When the loss function $\lsc$ depends only on $h(x, \cdot)$ for all
$h$, $x$, and $y\in \sY$, that is, $\sfL(h, x, y) = \Psi(h(x, 1),
\ldots, h(x, n + 1), y)$ for some function $\Psi$, it can be shown
that the minimizability gap vanishes for the family of all measurable
functions: $\sM( \sH_{\rm{all}}) = 0$
\citep[lemma~2.5]{steinwart2007compare}. However, in general, the
minimizability gap is non-zero for restricted hypothesis sets $\sH$
and is therefore essential to analyze. It is worth noting that the
minimizability gap can be upper bounded by the approximation error
$\sA_{\lsc}(\sH) = \sE^*_{\lsc}(\sH) - \E_x\bracket[\big]{\inf_{ h \in
    \sH_{\rm{all}}} \E_y \bracket{\lsc( h, X, y) \mid X =
    x}}$. However, the minimizability gap is a more refined quantity
than the approximation error and can lead to more favorable
guarantees (see Appendix~\ref{app:better-bounds}).

\section{Single-stage score-based formulation}
\label{sec:score}

In this section, we first derive the general form of a family of
surrogate loss functions $\lsc$ for $\labsc$ by analyzing the
abstention loss $\labsc$. Next, we give $ \sH$-consistency bounds for
these surrogate losses, which provide non-asymptotic hypothesis
set-specific guarantees upper-bounding the estimation error of the
loss function $\labsc$ in terms of estimation error of $\lsc$.

\subsection{General Surrogate Losses}
\label{sec:score-general}
Consider a hypothesis $ h$ in the score-based setting.  Note that
for any $(x, y) \in \sX \times \sY$, $ \hh(x) = n + 1$ implies $
\hh(x) \neq y$, therefore, we have: $\1_{ \hh(x)\neq y}\1_{
  \hh(x) = n + 1}= \1_{ \hh(x) = n + 1}$.  Thus,
$\labsc( h, x, y)$ can be rewritten as follows:
\begin{align*}
\labsc( h, x, y)
& = \1_{ \hh(x)\neq y}\paren*{1 - \1_{ \hh(x) = n + 1}} + c \1_{ \hh(x) = n + 1}\\
& = \1_{ \hh(x)\neq y} - \1_{ \hh(x)\neq y} \1_{ \hh(x) = n + 1} + c \1_{ \hh(x) = n + 1}\\
& = \1_{ \hh(x)\neq y} - \1_{ \hh(x) = n + 1} + c \1_{ \hh(x) = n + 1}\\
& = \1_{ \hh(x)\neq y} + (c - 1) \1_{ \hh(x) = n + 1}\\
& = \1_{ \hh(x)\neq y} + (1 - c) \1_{ \hh(x)\neq n + 1} + c - 1.
\end{align*}
In view of this expression, since the last term $(c - 1)$ is a
constant, if $ \ell$ is a surrogate loss for the zero-one
multi-class classification loss over the set of labels $ \sY$, then
$\lsc$ defined as follows is a natural surrogate loss for $\labsc$:
for all $(x, y) \in \sX \times \sY$,
\begin{equation}
\label{eq:sur-score}
\lsc \paren*{ h, x, y}
=  \ell \paren*{ h, x, y} + (1 - c) \,  \ell\paren*{ h, x, n + 1}.
\end{equation}
This is precisely the form of the surrogate losses proposed by
\citet{mozannar2020consistent}, for which the analysis just presented
gives a natural derivation. This is also the form of the surrogate
losses adopted by \citet{caogeneralizing}.

\subsection{$\sH$-Consistency Bounds Guarantees}
\label{sec:score-bounds}

\citet{caogeneralizing} presented a nice study of the surrogate loss
$\lsc$ for a specific family of zero-one loss surrogates $ \ell$.
The authors showed that the surrogate loss $\lsc$ is Bayes-consistent
with respect to the score-based abstention loss $\labsc$ when $
\ell$ is Bayes-consistent with respect to the multi-class zero-one
classification loss $\ell_{0-1}$.  Bayes-consistency guarantees that,
asymptotically, a nearly optimal minimizer of $\lsc$ over the family
of all measurable functions is also a nearly optimal minimizer of
$\labsc$. However, this does not provide any guarantee for a restricted
subset $ \sH$ of the family of all measurable functions.  It also
provides no guarantee for approximate minimizers since convergence
could be arbitrarily slow and the result is only asymptotic.

In the following, we will prove $\sH$-consistency bounds guarantees,
which are stronger results that are non-asymptotic and that hold for a
restricted hypothesis set $ \sH$. The specific instance of our results
where $\sH$ is the family of all measurable functions directly implies
the Bayes-consistency results of \citet{caogeneralizing}.

\paragraph{$\sH$-Consistency Bounds for Cross-Entropy Abstention Losses}

We first prove $\sH$-consistency bounds for a broad family of
score-based abstention surrogate losses $\lsc_{\mu}$, that we will
refer to as \emph{\compsum\ score-based surrogate losses}. These are
loss functions defined by
\begin{equation}
\label{eq:L-mu}
\lsc_{\mu} \paren*{ h, x, y}
=  \ell_{\mu} \paren*{ h, x, y} + (1 - c) \, \ell_{\mu}\paren*{ h, x, n + 1},    
\end{equation}
where, for any $ h\in  \sH$, $x\in \sX$, $y\in\sY$ and $\mu\geq 0$,
\begin{align*}
\ell_{\mu}(h,x, y) =\begin{cases}
\frac{1}{1 - \mu} \paren*{\bracket*{\sum_{y'\in\sY \cup \curl*{n+1}} e^{{ h(x, y') -  h(x, y)}}}^{1 - \mu} - 1} & \mu\neq 1  \\
\log\paren*{\sum_{y'\in \sY\cup \curl*{n+1}} e^{ h(x, y') -  h(x, y)}} & \mu = 1.
\end{cases}    
\end{align*}
The loss function $\ell_{\mu}$ coincides with the (multinomial)
logistic loss
\citep{Verhulst1838,Verhulst1845,Berkson1944,Berkson1951} when
$\mu=1$, matches the generalized cross-entropy loss
\citep{zhang2018generalized} when $\mu\in (1,2)$, and the mean
absolute loss \citep{ghosh2017robust} when $\mu=2$. Thus, the
\compsum\ score-based surrogate losses $ \sfL_{\mu}$ include the
abstention surrogate losses proposed in \citep{mozannar2020consistent}
which correspond to the special case of $\mu=1$ and the abstention
surrogate losses adopted in \citep{caogeneralizing}, which correspond
to the special case of $\mu\in [1, 2]$.

We say that a hypothesis set $\sH$ is \emph{symmetric} when the
scoring functions it induces do not depend on any particular
ordering of the labels, that is when there exists a family $\sF$ of
functions $f$ mapping from $\sX$ to $\Rset$ such that, for any $x \in
\sX$, $\curl*{\bracket*{h(x,1),\ldots,h(x,n),h(x,n+1)}\colon h\in \sH}
= \curl*{\bracket*{f_1(x),\ldots, f_n(x),f_{n+1}(x)}\colon f_1,
  \ldots, f_{n+1}\in \sF}$. We say that a hypothesis set $\sH$ is
\emph{complete} if the set of scores it generates spans $\Rset$, that
is, $\curl*{h(x, y)\colon h\in \sH} = \Rset$, for any $(x, y)\in \sX
\times \sY \cup \curl*{n+1}$. Common hypothesis sets used in practice,
such as the family of linear models, that of neural networks and of
course that of all measurable functions are all symmetric and
complete.  The guarantees given in the following result are thus
general and widely applicable.

\begin{restatable}[\textbf{$\sH$-consistency bounds for \compsum\
      score-based surrogates}]
  {theorem}{BoundCompSum}
\label{Thm:bound_comp_sum}
Assume that $\sH$ is symmetric and complete. Then, for any hypothesis
$h \in \sH$ and any distribution $\sD$, the following inequality holds:
\ifdim\columnwidth=\textwidth
\begin{equation*}
\sE_{\labsc}( h) - \sE_{\labsc}^*( \sH) + \sM_{\labsc}( \sH)
\leq \Gamma_{\mu}
  \paren*{\sE_{\lsc_{\mu}}( h) - \sE_{\lsc_{\mu}}^*(\sH)
    + \sM_{\lsc_{\mu}}(\sH)},
\end{equation*}
\else
\begin{multline*}
 \sE_{\labsc}( h) - \sE_{\labsc}^*( \sH) + \sM_{\labsc}( \sH)\\
\leq \Gamma_{\mu}
  \paren*{\sE_{\lsc_{\mu}}( h) - \sE_{\lsc_{\mu}}^*(\sH)
    + \sM_{\lsc_{\mu}}(\sH)},
\end{multline*}
\fi
where $\Gamma_{\mu}(t)=\begin{cases}
\sqrt{(2-c)2^{\mu}(2-\mu) t} & \mu\in [0,1)\\
\sqrt{2(2-c)(n+1)^{\mu-1}t } & \mu\in [1,2) \\
(\mu - 1)(n+1)^{\mu - 1} t & \mu \in [2,\plus \infty).
\end{cases}$
\end{restatable}
The proof is given in Appendix~\ref{app:bound_comp_sum}. It consists
of\ignore{using the general $\sH$-consistency theorems of
  \cite{awasthi2022Hconsistency,awasthi2022multi},} analyzing the
calibration gap of the score-based abstention loss $\labsc$ and that
of $\lsc_{\mu}$, and of finding a concave function $\Gamma_\mu$
relating these two quantities. Note that our proofs and results are
distinct, original, and more complex than those in the standard
setting \citep{MaoMohriZhong2023cross}, where the standard loss
$\ell_{\mu}$ is analyzed. Establishing $\sH$-consistency bounds for
$\lsc_{\mu}$ is more intricate compared to $\ell_{\mu}$. This is
because the target loss in the score-based multi-class abstention is
inherently different from that of the standard multi-class scenario
(the multi-class zero-one loss). Thus, we need to tackle a more
complex calibration gap, integrating both the conditional probability
vector and the cost function. This complexity presents an added layer
of challenge when attempting to establish a lower bound for the
calibration gap of the surrogate loss in relation to the target loss
in the score-based abstention setting.

To understand the result, consider first the case where the
minimizability gaps are zero. As mentioned earlier, this would be the
case, for example, when $ \sH$ is the family of all measurable
functions or when $ \sH$ contains the Bayes classifier.
In that case, the theorem shows that if the estimation loss
$(\sE_{\lsc_{\mu}}( h) - \sE_{\lsc_{\mu}}^*(\sH))$ is reduced to $\e$,
then, for $\mu \in [0, 2)$, in particular for the logistic score-based
  surrogate ($\mu = 1$) and the generalized cross-entropy score-based
  surrogate ($\mu \in (1, 2)$), modulo a multiplicative constant, the
  score-based abstention estimation loss $(\sE_{\labsc}( h) -
  \sE_{\labsc}^*( \sH))$ is bounded by $\sqrt{\e}$. The bound is even
  more favorable for the mean absolute error score-based surrogate
  ($\mu = 2$) or for \compsum\ score-based surrogate $\lsc_{\mu}$ with
  $\mu \in (2, +\infty)$ since in that case, modulo a multiplicative
  constant, the score-based abstention estimation loss $(\sE_{\labsc}(
  h) - \sE_{\labsc}^*( \sH))$ is bounded by $\e$.
  
These are strong results since they are not asymptotic and are
hypothesis set-specific. In particular,
Theorem~\ref{Thm:bound_comp_sum} provides stronger guarantees than the
Bayes-consistency results of \cite{mozannar2020consistent} or
\cite{caogeneralizing} for cross-entropy abstention surrogate losses
\eqref{eq:L-mu} with the logistic loss ($\mu = 1$), generalized
cross-entropy loss ($\mu\in (1,2)$) and mean absolute error loss
($\mu=2$) adopted for $\ell$. These Bayes-consistency results can be
obtained by considering the special case of $ \sH$ being the family of
all measurable functions and taking the limit.
  
Moreover, Theorem~\ref{Thm:bound_comp_sum} also provides similar
guarantees for other types of \compsum\ score-based surrogate losses,
such as $\mu\in [0,1)$ and $\mu\in [2,\plus\infty)$, which are new
    surrogate losses for score-based multi-class abstention that, to
    the best of our knowledge, have not been previously studied in the
    literature. In particular, our $\sH$-consistency bounds can help
    theoretically compare different \compsum\ score-based surrogate
    losses and guide the design of a multi-class abstention
    algorithm. In contrast, asymptotic consistency guarantees given
    for a subset of \compsum\ score-based surrogate losses in
    \citep{mozannar2020consistent,caogeneralizing} do not provide any
    such comparative information.
    
    Recall that the minimizability gap is always upper bounded by the
    approximation error. By Lemma~\ref{lemma:calibration_gap_score} in
    Appendix~\ref{app:score}, the minimizability gap for the
    abstention loss $\sM_{\labs}(\sH)$ coincides with the
    approximation error $\sA_{\labs}(\sH)$ when the labels generated
    by the hypothesis set encompass all possible outcomes, which
    naturally holds true for typical hypothesis sets. However, for a
    surrogate loss, the minimizability gap is in general a more
    refined quantity than the approximation error and can lead to more
    favorable guarantees. More precisely, $\sH$-consistency bounds
    expressed in terms of minimizability gaps are better and more
    significant than the excess error bounds expressed in terms of
    approximation errors (See Appendix~\ref{app:better-bounds} for a
    more detailed discussion).

\subsection{Analysis of Minimizability Gaps}
\label{sec:min-gaps}

In general, the minimizability gaps do not vanish and their magnitude,
$\sM_{\lsc_{\mu}}(\sH)$, is important to take into account when
comparing \compsum\ score-based surrogate losses, in addition to the
functional form of $\Gamma_\mu$. Thus, we will specifically analyze
them below.  Note that the dependency of the multiplicative constant
on the number of classes in some of these bounds ($\mu \in (1,
+\infty)$) makes them less favorable, while for $\mu \in [0, 1]$, the
bounds do not depend on the number of classes.

\ignore{The minimizability gap $\sM_{\lsc}( \sH) = \sE^*_{\lsc}( \sH)
  - \E_x \bracket[\big]{\inf_{ h \in \sH} \E_y\bracket*{\lsc( h, X, y)
      \mid X = x}}$ measures the difference between the best-in-class
  expected loss and the expected infimum of the pointwise expected
  loss. }
In the deterministic cases where for any $x\in \sX$ and $y\in \sY$,
either $p(x, y)=0$ or $1$, the pointwise expected loss admits an
explicit form. Thus, the following result characterizes the
minimizability gaps directly in those cases.  \ignore{We will consider
  the deterministic case where for any $x\in \sX$ and $y\in \sY$,
  either $p(x, y)=0$ or $1$.  We will specifically study the
  $\Lambda$-bounded hypothesis sets $ \sH_{\Lambda}$, that is, for
  fixed $x \in \sX$, we have $\curl*{\paren*{ h(x, 1), \ldots, h(x,
      n), h(x, n+1)}\colon h \in \sH}$ = $[-\Lambda, +\Lambda]^{n+1}$
  where $\Lambda \in [0,\plus \infty]$. The family of all measurable
  functions is a special $\Lambda$-bounded hypothesis set which
  corresponds to $\Lambda=\plus \infty$. The following theorem
  characterizes the minimizability gaps in those cases.
\begin{restatable}[\textbf{Characterization of minimizability gaps}]
  {theorem}{GapUpperBoundDetermi}
\label{Thm:gap-upper-bound-determi}
Assume that $ \sH$ is $\Lambda$-bounded. Then, for the
\compsum\ score-based surrogate losses $\lsc_{\mu}$ and any
deterministic distribution, the minimizability gaps can be upper
bounded as $\sM_{\lsc_{\mu}}( \sH) \leq
\sM_{\lsc_{\mu}}(\sH)(\Lambda)$, where
\begin{align*}
 \sM_{\lsc_{\mu}}(\sH)(\Lambda)=
 \min\curl*{\sT_{\mu}\paren*{\sE^*_{\lsc_{0}}( \sH)},c}
 - \min\curl*{\sT_{\mu}\paren*{\uv \sE^*(\sH)},c}
\end{align*}
with $\uv \sE^*(\sH)=e^{-2 \Lambda}\,n\leq \sE^*_{\lsc_{0}}( \sH)$ and
$\sT_{\mu}(t)$ defined as
\begin{equation*}
\sT_{\mu}(t)
=
\begin{cases}
\frac{1}{1 - \mu} \paren*{(1 + t)^{1 - \mu} - 1} & \mu \geq 0, \mu \neq 1 \\
\log(1 + t) & \mu = 1.
\end{cases}
\end{equation*}
\end{restatable}
See Appendix~\ref{app:score} for the proof.
}

\begin{restatable}[\textbf{Characterization of minimizability gaps}]
  {theorem}{GapUpperBoundDetermi}
\label{Thm:gap-upper-bound-determi}
Assume that $ \sH$ is symmetric and complete. Then, for the \compsum\ score-based surrogate losses $\lsc_{\mu}$ and any deterministic distribution, the minimizability gaps can be characterized as follows:
\begin{align*}
\sM_{\lsc_{\mu}}( \sH) = \sE_{\lsc_{\mu}}^*(\sH) -  \begin{cases}
\frac{1}{1 - \mu} \bracket*{\bracket*{1+\paren*{1-c}^{\frac{1}{2-\mu}}}^{2 - \mu} \mspace{-20mu} - (2-c)} & \mu \notin \curl*{1,2}\\
-\log \paren*{\frac{1}{2-c}}-(1-c)\log \paren*{\frac{1-c}{2-c}}
 & \mu=1\\
1-c & \mu =2.
\end{cases}
\end{align*}
\end{restatable}
See Appendix~\ref{app:gap-upper-bound-determi} for the proof. By
l’H\^opital's rule, $\sE_{\lsc_{\mu}}^*(\sH)-\sM_{\lsc_{\mu}}( \sH)$
is continuous as a function of $\mu$ at $\mu = 1$. In light of the
equality $\lim_{x\to 0^{+}}\paren[\big]{1 + u^{\frac1x}}^x =
\max\curl*{1, u} = 1$, for $u \in [0, 1]$,
$\sE_{\lsc_{\mu}}^*(\sH)-\sM_{\lsc_{\mu}}( \sH)$ is continuous as a
function of $\mu$ at $\mu = 2$. Moreover, for any $c\in (0, 1)$,
$\sE_{\lsc_{\mu}}^*(\sH)-\sM_{\lsc_{\mu}}( \sH)$ is decreasing with
respect to $\mu$. On the other hand, since the function $\mu \mapsto
\frac{1}{1-\mu}\paren*{t^{1-\mu}-1} \1_{\mu \neq 1} + \log(t) \1_{\mu
  = 1}$ is decreasing for any $t > 0$, we obtain that $ \ell_{\mu}$ is
decreasing with respect to $\mu$, which implies that $\lsc_{\mu}$ is
decreasing and then $\sE_{\lsc_{\mu}}^*(\sH)$ is decreasing with
respect to $\mu$ as well. For a specific problem, a favorable
$\mu\in[0, \infty)$ is one that minimizes $\sM_{\lsc_{\mu}}( \sH)$,
  which, in practice, can be selected via cross-validation.

\subsection{General Transformation}

More generally, we prove the following result, which shows that an
$\sH$-consistency bound for $ \ell$ with respect to the
zero-one loss, yields immediately an $\sH$-consistency bound
for $\lsc$ with respect to $\labsc$.

\begin{restatable}{theorem}{BoundScore}
\label{Thm:bound-score}
Assume that $ \ell$ admits an $ \sH$-consistency bound with respect to
the multi-class zero-one classification loss $ \ell_{0-1}$ with a
concave function $\Gamma$, that is, for all $ h \in \sH$, the
following inequality holds: \ifdim\columnwidth=\textwidth
\begin{equation*}
\sE_{\ell_{0-1}}( h) - \sE_{\ell_{0-1}}^*( \sH) + \sM_{\ell_{0-1}}( \sH)
\leq \Gamma\paren*{\sE_{ \ell}( h)-\sE_{ \ell}^*( \sH) +\sM_{ \ell}( \sH)}.
\end{equation*}
\else\begin{multline*}
\sE_{\ell_{0-1}}( h) - \sE_{\ell_{0-1}}^*( \sH) + \sM_{\ell_{0-1}}( \sH)\\
\leq \Gamma\paren*{\sE_{ \ell}( h)-\sE_{ \ell}^*( \sH) +\sM_{ \ell}( \sH)}.
\end{multline*}
\fi Then, $\lsc$ defined by \eqref{eq:sur-score} admits an
$\sH$-consistency bound with respect to $\labsc$ with the functional
form $(2 - c)\Gamma\paren{\frac{t}{2 - c}}$, that is, for all $ h\in
\sH$, we have \ifdim\columnwidth=\textwidth
\begin{equation*}
\sE_{\labsc}( h) - \sE_{\labsc}^*( \sH) + \sM_{\labsc}( \sH)
\leq (2 - c) \Gamma\paren*{\frac{\sE_{\lsc}( h) - \sE_{\lsc}^*( \sH) +\sM_{\lsc}( \sH)}{2 - c}}.
\end{equation*}
\else
\begin{multline*}
\sE_{\labsc}( h) - \sE_{\labsc}^*( \sH) + \sM_{\labsc}( \sH)\\
\leq (2 - c) \Gamma\paren*{\frac{\sE_{\lsc}( h) - \sE_{\lsc}^*( \sH) +\sM_{\lsc}( \sH)}{2 - c}}.
\end{multline*}
\fi
\end{restatable}
The proof is given in Appendix~\ref{app:bound-score}.
\cite{awasthi2022multi} recently presented a series of results
providing $\sH$-consistency bounds for common surrogate losses in the
standard multi-class classification, including max losses such as
those of \citet{crammer2001algorithmic}, sum losses such as those of
\citet{weston1998multi} and constrained losses such as the loss
functions adopted by \citet{lee2004multicategory}. Thus, plugging in
any of those $\sH$-consistency bounds in Theorem~\ref{Thm:bound-score}
yields immediately a new $ \sH$-consistency bound for the
corresponding score-based abstention surrogate losses.

\section{Two-stage score-based formulation}
\label{sec:two-stage}

In the single-stage scenario discussed in Section~\ref{sec:score}, the
learner simultaneously learns when to abstain and how to make
predictions otherwise. However, in practice often there is already a
predictor available and retraining can be very costly. A two-stage
solution is thus much more relevant for those critical applications,
where the learner only learns when to abstain in the second stage
based on the predictor trained in the first stage. With the two stage
solution, we can improve the performance of a large pre-trained model
by teaching it the option of abstaining without having to retrain the
model. In this section, we analyze a two-stage algorithmic scheme, for
which we propose surrogate losses that we prove to benefit from
$\sH$-consistency bounds.

Given a hypothesis set $\sH$ of functions mapping from $\sX \times
(\sY \cup \curl*{n + 1})$ to $\Rset$, it can be decomposed into $ \sH=
\sH_{\sY}\times \sH_{n+1}$, where $\sH_{\sY}$ denotes the hypothesis
set spanned by the first $n$ scores corresponding to the labels, and
$\sH_{n+1}$ represents the hypothesis set spanned by the last score
corresponding to the additional category.\ignore{The corresponding
  hypotheses are denoted by $h_{\sY} \in \sH_{\sY}$ and $h_{n+1} \in
  \sH_{n+1}$, respectively.}  We consider the following two-stage
algorithmic scheme: in the first stage, we learn a hypothesis $h_{\sY}
\in \sH_{\sY}$ by optimizing a surrogate loss $\ell$ for standard
multi-class classification; in the second stage, we fix the $h_{\sY}$
learned in the first stage and then learn a hypothesis $h_{n+1} \in
\sH_{n+1}$ by optimizing a surrogate loss function $\ell_{h_{\sY}}$
defined for any $h_{n+1} \in \sH_{n+1}$ and $(x, y) \in \sX \times
\sY$ by
\begin{align}
\label{eq:ell-Phi-h}
\ell_{h_{\sY}}\paren*{h_{n+1}, x, y}  
= \1_{\hh_{\sY}(x) \neq y} \Phi\paren*{h_{n+1}(x) - \max_{y\in \sY}h_{\sY}(x, y)} + c \Phi\paren*{\max_{y\in \sY}h_{\sY}(x, y) - h_{n+1}(x)},
\end{align}
where $\Phi$ is a decreasing function.  The learned hypothesis $h \in
\sH$ corresponding to those two stages can be expressed as $h =
(h_{\sY}, h_{n+1})$.  We note that the first stage consists of the
familiar task of finding a predictor using a standard surrogate loss
such as the logistic loss $\ell(h,x, y)=\log\paren*{\sum_{y'\in
    \sY}e^{h(x, y')-h(x, y)}}$ (or cross-entropy combined with the
softmax). Recall that the learner abstains from making a prediction
for $x$ and incurs a cost $c$ when $ h_{n+1}(x) \geq \max_{y \in \sY}
h_{\sY}(x, y)$.  In the second stage, the first term of
\eqref{eq:ell-Phi-h} encourages abstention for an input instance whose
prediction made by the pre-trained predictor $h_{\sY}$ is incorrect,
while the second term penalizes abstention according to the cost
$c$. The function $\Phi$ can be chosen as any margin-based loss
function in binary classification, including the exponential loss or
the logistic loss.

Let $\ell_{0-1}^{\rm{binary}}$ be the binary zero-one classification loss. Then, the two-stage surrogate losses benefit from the $\sH$-consistency bounds shown in Theorem~\ref{Thm:bound-general-two-step}. For a fixed parameter $\tau$, we define the $\tau$-translated hypothesis set of $\sH_{n+1}$ by $\sH_{n+1}^{\tau} =\curl*{ h_{n+1} - \tau : h_{n+1} \in \sH_{n+1}}$.
\begin{restatable}[\textbf{$\sH$-consistency bounds for
      two-stage surrogates}]{theorem}{BoundGenralTwoStep}
\label{Thm:bound-general-two-step}
Given a hypothesis set $\sH=\sH_{\sY}\times \sH_{n+1}$. Assume that
$\ell$ admits an $\sH_{\sY}$-consistency bound with respect to the
multi-class zero-one classification loss $ \ell_{0-1}$ and that $\Phi$
admits an $\sH_{n+1}^{\tau}$-consistency bound with respect to the
binary zero-one classification loss $\ell_{0-1}^{\rm{binary}}$ for any
$\tau \in \Rset$.  Thus, there are non-decreasing concave functions
$\Gamma_1$ and $\Gamma_2$ such that, for all $h_{\sY}\in \sH_{\sY}$,
$h_{n+1}^{\tau} \in \sH_{n+1}^{\tau}$ and $\tau \in \Rset$, we have
\begin{align*}
 \sE_{\ell_{0-1}}(h_{\sY}) - \sE_{\ell_{0-1}}^*(\sH_{\sY}) + \sM_{\ell_{0-1}}(\sH_{\sY}) &\leq \Gamma_1\paren*{\sE_{\ell}(h_{\sY})-\sE_{\ell}^*(\sH_{\sY}) +\sM_{\ell}(\sH_{\sY})}\\
 \sE_{\ell_{0-1}^{\rm{binary}}}(h_{n+1}^{\tau}) - \sE_{\ell_{0-1}^{\rm{binary}}}^*(\sH_{n+1}^{\tau}) + \sM_{\ell_{0-1}^{\rm{binary}}}(\sH_{n+1}^{\tau}) &\leq \Gamma_2\paren*{\sE_{\Phi}(h_{n+1}^{\tau})-\sE_{\Phi}^*(\sH_{n+1}^{\tau}) +\sM_{\Phi}(\sH_{n+1}^{\tau})}.
\end{align*}
Then, the following holds for all $ h=(h_{\sY},h_{n+1})\in \sH$:
\begin{align*}
\sE_{\labs}(h) - \sE_{\labs}^*(\sH) + \sM_{\labs}(\sH)
& \leq \Gamma_1\paren*{\sE_{\ell}(h_{\sY})-\sE_{\ell}^*(\sH_{\sY}) +\sM_{\ell}(\sH_{\sY})}\\
& \quad + (1+c)\Gamma_2\paren[\bigg]{\frac{\sE_{\ell_{h_{\sY}}}(h_{n+1})-\sE_{\ell_{h_{\sY}}}^*(\sH_{n+1}) +\sM_{\ell_{h_{\sY}}}(\sH_{n+1})}{c}},
\end{align*}
where the
constant factors $(1 + c)$ and $\frac{1}{c}$ can be removed 
when $\Gamma_2$ is linear.
\end{restatable}
The proof is given in Appendix~\ref{app:bound-general-two-step}. The
assumptions in Theorem~\ref{Thm:bound-general-two-step} are mild and
hold for common hypothesis sets such as linear models and neural
networks with common surrogate losses in the binary and multi-class
classification, as shown by
\citep{awasthi2022Hconsistency,awasthi2022multi}. Recall that the
minimizability gaps vanish when $\sH_{\sY}$ and $\sH_{n+1}$ are the
family of all measurable functions or when $\sH_{\sY}$ and $\sH_{n+1}$
contain the Bayes predictors.  In their absence, the theorem shows
that if the estimation loss
$(\sE_{\ell}(h_{\sY})-\sE_{\ell}^*(\sH_{\sY}))$ is reduced to $\e_1$
and the estimation loss
$(\sE_{\ell_{h_{\sY}}}(h_{n+1})-\sE_{\ell_{h_{\sY}}}^*(\sH_{n+1}))$ to
$\e_2$, then, modulo constant factors, the score-based abstention
estimation loss $(\sE_{\labs}(h) - \sE_{\labs}^*(\sH))$ is bounded by
$\Gamma_1(\e_1) + \Gamma_2(\e_2)$.
Thus, this gives a strong guarantee for the surrogate losses described
in this two-stage setting.

\section{Realizable $\sH$-consistency and benefits of two-stage surrogate losses}
\label{sec:realizable}

\citet{pmlr-v206-mozannar23a} recently showed that cross-entropy
score-based surrogate losses are not realizable $\sH$-consistent, as
defined by \citet{long2013consistency,zhang2020bayes}, in relation to
abstention loss. Instead, the authors proposed a novel surrogate
loss that is proved to be realizable $\sH$-consistent when $\sH$ is
\emph{closed under scaling}, although its Bayes-consistency remains
unclear. Devising a surrogate loss that exhibits both
Bayes-consistency and realizable $\sH$-consistency remains an open
problem. A hypothesis set $\sH$ is said to be \emph{closed under
scaling} if, for any hypothesis $h$ belonging to $\sH$, the scaled
hypothesis $\alpha h$ also belongs to $\sH$ for all $\alpha \in \Rset$.

We prove in Theorem~\ref{Thm:bound-general-two-step-realizable} of
Appendix~\ref{app:bound-general-two-step-realizable}, that for any
realizable distribution, when both the first-stage surrogate
estimation loss $\sE_{\ell}(h_{\sY}) - \sE_{\ell}^*(\sH_{\sY})$ and
the second-stage surrogate estimation loss
$\sE_{\ell_{h_{\sY}}}(h_{n+1}) - \sE_{\ell_{h_{\sY}}}^*(\sH_{n+1})$
converge to zero, the abstention estimation loss $\sE_{\labs}(h) -
\sE_{\labs}^*(\sH)$ also approaches zero. This implies that the
two-stage score-based surrogate loss is realizable $\sH$-consistent
with respect to $\labsc$, which provides a significant advantage over
the single-stage cross-entropy score-based surrogate loss. It is
important to note that Theorem~\ref{Thm:bound-general-two-step} shows
that the two-stage formulation is also Bayes-consistent. This
addresses the open problem in \citep{pmlr-v206-mozannar23a} and
highlights the benefits of the two-stage formulation. In the following
section, our empirical results further demonstrate that the two-stage
score-based surrogate loss outperforms the state-of-the-art
cross-entropy score-based surrogate loss.

\section{Finite sample guarantees}
\label{sec:finite-sample}

Our $\sH$-consistency bounds enable the direct derivation of
finite-sample estimation bounds for a surrogate loss minimizer. These
are expressed in terms of the Rademacher complexity of the hypothesis
set $\sH$, the loss function, and the minimizability gaps. Here, we
provide a simple illustration based on
Theorem~\ref{Thm:bound_comp_sum}.

Let $\h h_S$ be the empirical minimizer of the surrogate loss
$\lsc_{\mu}$: $ \h h_S = \argmin_{h \in \sH} \frac{1}{m} \sum_{i =
  1}^m \lsc_{\mu}(h, x_i, y_i)$, for an i.i.d sample $S =
\paren*{(x_1, y_1), \ldots, (x_m, y_m)}$ of size $m$. Let
$\Rad_m^{\lsc_{\mu}}(\sH)$ be the Rademacher complexity of the set
$\sH_{\lsc_{\mu}} = \curl*{(x, y) \mapsto \lsc_{\mu}(h, x, y) \colon h
  \in \sH}$ and $B_{\lsc_{\mu}}$ an upper bound on the surrogate loss
$\lsc_{\mu}$. By using the standard Rademacher complexity bounds
\citep{MohriRostamizadehTalwalkar2018}, for any $\delta>0$, with
probability at least $1 - \delta$, the following holds for all $h \in
\sH$:
\[\abs*{\sE_{\lsc_{\mu}}(h) - \h \sE_{\lsc_{\mu}, S}(h)}
\leq 2 \Rad_m^{\lsc_{\mu}}(\sH) +
B_{\lsc_{\mu}} \sqrt{\tfrac{\log (2/\delta)}{2m}}.\]
Fix $\e > 0$. By the definition of the infimum, there exists $h^* \in
\sH$ such that $\sE_{\lsc_{\mu}}(h^*) \leq
\sE_{\lsc_{\mu}}^*(\sH) + \e$. By definition of
$\h h_S$, we have
\begin{align*}
\sE_{\lsc_{\mu}}(\h h_S) - \sE_{\lsc_{\mu}}^*(\sH)
& = \sE_{\lsc_{\mu}}(\h h_S) - \h\sE_{\lsc_{\mu}, S}(\h h_S) + \h\sE_{\lsc_{\mu}, S}(\h h_S) - \sE_{\lsc_{\mu}}^*(\sH)\\
& \leq \sE_{\lsc_{\mu}}(\h h_S) - \h\sE_{\lsc_{\mu}, S}(\h h_S) + \h\sE_{\lsc_{\mu}, S}(h^*) - \sE_{\lsc_{\mu}}^*(\sH)\\
& \leq \sE_{\lsc_{\mu}}(\h h_S) - \h\sE_{\lsc_{\mu}, S}(\h h_S) + \h\sE_{\lsc_{\mu}, S}(h^*) - \sE_{\lsc_{\mu}}^*(h^*) + \e\\
& \leq
  2 \bracket*{2 \Rad_m^{\lsc_{\mu}}(\sH) +
B_{\lsc_{\mu}} \sqrt{\tfrac{\log (2/\delta)}{2m}}} + \e.    
\end{align*}
Since the inequality holds for all $\e > 0$, it implies:
\[
\sE_{\lsc_{\mu}}(\h h_S) - \sE_{\lsc_{\mu}}^*(\sH)
\leq 
4 \Rad_m^{\lsc_{\mu}}(\sH) +
2 B_{\lsc_{\mu}} \sqrt{\tfrac{\log (2/\delta)}{2m}}.
\]
Plugging in this inequality in the bound of
Theorem~\ref{Thm:bound_comp_sum}, we obtain that for any $\delta > 0$,
with probability at least $1 - \delta$ over the draw of an i.i.d
sample $S$ of size $m$, the following finite sample guarantee holds
for $\h h_S$:
\begin{align*}
\sE_{\labsc}(\h h_S) - \sE_{\labsc}^*( \sH) \leq \Gamma_{\mu}
  \paren[\Big]{4 \Rad_m^{\lsc_{\mu}}(\sH)
  +
2 B_{\lsc_{\mu}} \textstyle \sqrt{\tfrac{\log \frac{2}{\delta}}{2m}}
    + \sM_{\lsc_{\mu}}(\sH)} - \sM_{\labsc}( \sH).    
\end{align*}
To our knowledge, these are the first abstention estimation loss
guarantees for empirical minimizers of a cross-entropy score-based
surrogate loss. Our comments about the properties of $\Gamma_{\mu}$
below Theorem~\ref{Thm:bound_comp_sum}, in particular its functional
form or its dependency on the number of classes $n$, similarly apply
here. Similar finite sample guarantees can also be derived based on
Theorems~\ref{Thm:bound-score} and \ref{Thm:bound-general-two-step}.

As commented before Section~\ref{sec:min-gaps}, for a surrogate loss,
the minimizability gap is in general a more refined quantity than the
approximation error, while for the abstention loss, these two
quantities coincide for typical hypothesis sets (See
Appendix~\ref{app:better-bounds}). Thus, our bound can be rewritten as
follows for typical hypothesis sets:
\begin{align*}
\sE_{\labsc}(\h h_S) - \sE_{\labsc}^*( \sH_{\rm{all}}) \leq \Gamma_{\mu}
  \paren*{4 \Rad_m^{\lsc_{\mu}}(\sH) +
    2 B_{\lsc_{\mu}}
    \sqrt{\frac{\log \frac{2}{\delta}}{2m}}
    + \sM_{\lsc_{\mu}}(\sH)}.
\end{align*}
Our guarantee is thus more favorable and more relevant than a similar
finite sample guarantee where $\sM_{\lsc_{\mu}}( \sH)$ is replaced
with $\sA_{\lsc_{\mu}}(\sH)$, which could be derived from an excess
error bound.

\section{Experiments}
\label{sec:experiments}

In this section, we report the results of experiments comparing the
single-stage and two-stage score-based abstention surrogate losses,
for three widely used datasets CIFAR-10, CIFAR-100
\citep{Krizhevsky09learningmultiple} and SVHN \citep{Netzer2011}.

\paragraph{Experimental Settings} 

As with \citep{mozannar2020consistent,caogeneralizing}, we use ResNet
\citep{he2016deep} and WideResNet (WRN) \citep{zagoruyko2016wide} with
ReLU activations. Here, ResNet-$n$ denotes a residual network with $n$
convolutional layers and WRN-$n$-$k$ denotes a residual network with
$n$ convolutional layers and a widening factor $k$. We trained
ResNet-$34$ for CIFAR-10 and SVHN, and WRN-$28$-$10$ for CIFAR-100.
We applied standard data augmentations, 4-pixel padding with $32
\times 32$ random crops and random horizontal flips for CIFAR-10 and
CIFAR-100. We used Stochastic Gradient Descent (SGD) with Nesterov
momentum \citep{nesterov1983method} and set batch size $1\mathord,024$
and weight decay $1\times 10^{-4}$ in the training. We trained for
$200$ epochs using the cosine decay learning rate schedule
\citep{loshchilov2016sgdr} with the initial learning rate of $0.1$.

For each dataset, the cost value $c$ was selected to be close to the
best-in-class zero-one classification loss, which are $\curl*{0.05,
  0.15, 0.03}$ for CIFAR-10, CIFAR-100 and SVHN respectively, since a
too small value leads to abstention on almost all points and a too
large one leads to almost no abstention. Other neighboring values for
$c$ lead to similar results.

The abstention surrogate loss proposed in
\citep{mozannar2020consistent} corresponds to the special case of
\compsum\ score-based surrogate losses $ \sfL_{\mu}$ with $\mu = 1$,
and meanwhile the abstention surrogate loss adopted in
\citep{caogeneralizing} corresponds to the special case of
\compsum\ score-based surrogate losses $ \sfL_{\mu}$ with $\mu =
1.7$. Note that the simple confidence-based approach by thresholding
estimators of conditional probability typically does not perform as
well as these state-of-the-art surrogate losses
\citep{caogeneralizing}. For our two-stage score-based abstention
surrogate loss, we adopted the logistic loss in the first stage and
the exponential loss $\Phi(t) = \exp(-t)$ in the second stage.

\paragraph{Evaluation} 

We evaluated all the models based on the abstention loss $\labsc$, and
reported the mean and standard deviation over three trials.

\paragraph{Results}

\begin{table*}[t]
\caption{Abstention Loss for Models Obtained with Different Surrogate
  Losses; Mean $\pm$ Standard Deviation\ignore{ over three runs} for
  Both Two-Stage Score-Based Abstention Surrogate Loss and The
  State-Of-The-Art Cross-Entropy Score-Based Surrogate Losses in
  \citep{mozannar2020consistent} ($\mu = 1.0$) and
  \citep{caogeneralizing} ($\mu=1.7$).}
  \vskip -0.1in
    \label{tab:comparison}
\begin{center}
    \begin{tabular}{@{\hspace{0pt}}lll@{\hspace{0pt}}}
      METHOD & DATASET & ABSTENTION LOSS \\
    \midrule
    Cross-entropy score-based ($\mu=1.0$) & \multirow{3}{*}{CIFAR-10} & 4.48\% $\pm$ 0.10\% \\
    cross-entropy score-based ($\mu=1.7$)  & & 3.62\% $\pm$ 0.07\%  \\
    \textbf{Two-stage score-based}  & & \textbf{3.22\% \!$\pm$ 0.04\%}    \\
    \midrule
    Cross-entropy score-based ($\mu=1.0$) & \multirow{3}{*}{CIFAR-100} & 10.40\% $\pm$ 0.10\% \\
    Cross-entropy score-based ($\mu=1.7$)  & & 14.99\% $\pm$ 0.01\%\\
    \textbf{Two-stage score-based} & & \textbf{\phantom{0}9.54\% \!$\pm$ 0.07\%}   \\
    \midrule
    Cross-entropy score-based ($\mu=1.0$) & \multirow{3}{*}{SVHN} & 1.61\% $\pm$ 0.06\% \\
    Cross-entropy score-based ($\mu=1.7$) & & 2.16\% $\pm$ 0.04\%\\
    \textbf{Two-stage score-based}  & & \textbf{0.93\% \!$\pm$ 0.02\%}  \\
    \end{tabular}
\end{center}
    \vskip -0.3in
\end{table*}

Table~\ref{tab:comparison} shows that the two-stage score-based
surrogate losses consistently outperform the cross-entropy score-based
surrogate losses used in the state-of-the-art algorithms
\citep{mozannar2020consistent,caogeneralizing} for all the
datasets. Table~\ref{tab:comparison} also shows the relative
performance of the cross-entropy surrogate \eqref{eq:L-mu} with
$\ell_{\mu}$ adopted as the generalized cross-entropy loss ($\mu=1.7$)
and that with $\ell_{\mu}$ adopted as the logistic loss ($\mu=1.0$)
varies by the datasets.

As show in Section~\ref{sec:two-stage} and
Section~\ref{sec:realizable}, the two-stage surrogate losses benefit
from the guarantees of both realizable $\sH$-consistency and
Bayes-consistency while the cross-entropy surrogate loss does not
exhibit realizable $\sH$-consistency, as shown by
\citet{pmlr-v206-mozannar23a}. This explains the superior performance
of two-stage surrogate losses over the cross-entropy surrogate
loss. It is worth noting that the hypothesis set we used for each
dataset is sufficiently rich, and the experimental setup closely
resembles a realizable scenario.

As our theoretical analysis (Theorem~\ref{Thm:bound_comp_sum} and
Theorem~\ref{Thm:gap-upper-bound-determi}) suggests, the relative
performance variation between the cross-entropy surrogate loss with
$\mu = 1.0$ used in \citep{mozannar2020consistent} and the
cross-entropy surrogate loss with $\mu = 1.7$ used in
\citep{caogeneralizing} can be explained by the functional forms of
their $\sH$-consistency bounds and the magnitude of their
minimizability gaps. Specifically, the dependency of the
multiplicative constant on the number of classes in $\sH$-consistency
bounds (Theorem~\ref{Thm:bound_comp_sum}) for the cross-entropy
surrogate loss with $\mu = 1.7$ makes it less favorable when dealing
with a large number of classes, such as in the case of CIFAR-100.
This suggests that the recent observation made in
\citep{caogeneralizing} that the cross-entropy surrogate with $\mu =
1.7$ outperforms the one with $\mu = 1.0$ does not apply to the
scenario where the evaluation involves datasets like CIFAR-100.  For a
more comprehensive discussion of our experimental results, please
refer to Appendix~\ref{app:experimemts}.
\ignore{ This agrees with our theoretical analysis based on
  $\sH$-consistency bounds in Theorem~\ref{Thm:bound_comp_sum} and
  Theorem~\ref{Thm:gap-upper-bound-determi}, since both losses have
  the same square-root functional form while on CIFAR-10, the
  magnitude of the minimizability gap decreases with $\mu$ in light of
  the fact that $\sE_{\lsc_{\mu}}^*(\sH)$ is close for both losses,
  and since on SVHN and CIFAR-100, the dependency of the
  multiplicative constant on the number of classes appears for
  $\mu=1.7$, which makes it less favorable, particularly clear when
  $n$ is large (CIFAR-100).
}

\section{Conclusion}

Our comprehensive study of score-based multi-class abstention
introduced novel surrogate loss families with strong hypothesis
set-specific and non-asymptotic theoretical guarantees. Empirical
results demonstrate the practical advantage of these surrogate losses
and their derived algorithms. This work establishes a powerful
framework for designing new, more reliable abstention-aware algorithms
applicable across diverse domains.

\ignore{
We presented a detailed study of score-based multi-class
abstention. We introduced new families of surrogate losses within the
framework and provided strong theoretical guarantees for them, which
are specific to the hypothesis set and non-asymptotic. Our empirical
results further illustrate the practical significance of these
surrogate losses and the new algorithms based on them. We believe that
our analysis can be leveraged in a wide range of scenarios to design
new algorithms.
}


\bibliography{mabsc}

\newpage
\appendix

\renewcommand{\contentsname}{Contents of Appendix}
\tableofcontents
\addtocontents{toc}{\protect\setcounter{tocdepth}{3}} 
\clearpage


\section{Related work}
\label{app:related-work}

The problem of abstention (or rejection) has been studied in several
publications in the past.  \citet{Chow1957,chow1970optimum} studied
the trade-off between error rate and rejection rate and also presented
an analysis of the Bayes optimal decision for this setting.  Later,
\cite{FumeraRoliGiacinto2000} suggested a multiple thresholds rule
when the a posteriori probabilities were affected by
errors. \citet{Tortorella2001} introduced an optimal rejection rule
for binary classifiers based on the Receiver Operating Characteristic
curve (ROC curve) and \citet{SantosPires2005} compared their approach
with that of \citet{chow1970optimum}.
A number of other publications suggested various rejection techniques
to decrease the misclassification rate, though without a theoretical
analysis \citep{FumeraRoli2002, Pietraszek2005,
  BounsiarGrallBeauseroy2007, LandgrebeTaxPaclikDuin2005, Melvin2008}.
Classification with a rejection option that incurs a cost was examined
by \citet{HerbeiWegkamp2005}, who gave the optimal rule for these
ternary functions.

\citet{el2010foundations} and \citet{wiener2011agnostic} proposed and
studied a framework for \emph{selective classification} based on a
classifier and a selector and an objective function defined as the
average loss on accepted samples, normalized by the average selection.
Several related connections with abstention have been studied,
including active learning
\citep{el2012active,wiener2015compression,wiener2015agnostic,
  puchkin2021exponential,denis2022active,zhu2022efficient}, rejection
in the multi-class setting
\citep{dubuisson1993statistical,tax2008growing,le2010optimum},
reinforcement learning \citep{li2008knows}, online learning
\citep{zhang2016extended}, modern confidence-based rejection
techniques \citep{geifman2017selective}, neural network architectures
for abstention \citep{geifman2019selectivenet}, loss functions derived
from the doubling rate of gambling \citep{ziyin2019deep},
disparity-free methods \citep{schreuder2021classification}, the
abstention problem within a "confidence set" framework
\citep{gangrade2021selective,chzhen2021set} and out-of-distribution
detection \citep{narasimhan2023learning}.

A standard method for abstention adopted in the past, which covers a
very large number of publications and dates back to the early work of
\citet{Chow1957,chow1970optimum}, is the so-called
\emph{confidence-based abstention}. This consists of first learning a
predictor and then abstaining when the score returned by the predictor
falls below some fixed threshold.  \citet{bartlett2008classification}
formulated a loss function for this setting taking into consideration
the abstention cost $c$ and suggested to learn a predictor using a
\emph{double hinge loss} that they showed benefits from consistency
results. Several other publications followed up on this approach
\citep{GrandvaletKeshetRakotomamonjyCanu2008,
  yuan2010classification,WegkampYuan2011}. \citet{yuan2010classification}
investigated the necessary and sufficient condition for consistency of
convex risk minimization with respect to the abstention loss and
obtained the corresponding excess error bounds in the same
setting. Other variants of this framework have also been studied in
\citep{lei2014classification, denis2020consistency}.

However, \citet*{CortesDeSalvoMohri2016,CortesDeSalvoMohri2023} argued
that, in general, confidence-based abstention is suboptimal, unless
the predictor learned is the Bayes classifier. They showed that, in
general, even in simple cases, no threshold-based abstention can
achieve the desired result. They introduced a novel framework for
abstention that consists of learning \emph{simultaneously} both a
predictor $h$ and a rejector $r$ that, in general, can be distinct
from a threshold-based function. They further defined a
\emph{predictor-rejector formulation} loss function for the pair $(h,
r)$, taking into consideration the abstention cost $c$. The authors
gave Rademacher complexity-based generalization bounds for this
learning problem. They also suggested several surrogate loss functions
for the abstention loss in the binary classification setting, and
further showed that these surrogate losses benefitted from consistency
guarantees. They designed algorithms based on these surrogate losses,
which they showed empirically outperform confidence-based abstention
baselines.  This work had multiple follow-up studies, including a
theoretical and algorithmic study of boosting with abstention
\citep{CortesDeSalvoMohri2016bis} and a study of the extension of the
results to multi-class setting \citep{NiCHS19}. These authors argued
that the design of calibrated or Bayes-consistent surrogate losses in
the multi-class classification setting based on the predictor-rejector
abstention loss of \citet{CortesDeSalvoMohri2016} was difficult and
left that as an open problem. Recently,
 \citet{MaoMohriZhong2023predictor} positively resolved this open problem by introducing new families of predictor-rejector surrogate losses, along with significantly stronger \emph{$\sH$-consistency bounds} guarantees. These are strong hypothesis
set-specific and non-asymptotic consistency guarantees for the surrogate losses, which
upper-bound the estimation error of the abstention loss function in
terms of the estimation error of the surrogate loss \citep{awasthi2022Hconsistency,awasthi2022multi,zheng2023revisiting,MaoMohriZhong2023cross,MaoMohriZhong2023characterization}. They have also been extended to the ranking setting \citep{MaoMohriZhong2023ranking,mao2023ranking}, structured prediction setting \citep{MaoMohriZhong2023structured}, regression setting \citep{mao2024regression}, top-$k$ classification setting \citep{mao2024top}, and adversarial classification setting \citep{goodfellow2014explaining,madry2017towards,tsipras2018robustness,carlini2017towards,awasthi2021calibration,awasthi2021finer,awasthi2021existence,awasthi2024dc} in recent work \citep{AwasthiMaoMohriZhong2023theoretically,MaoMohriZhong2023cross}.

\cite{cheng2023regression} applied the predictor-rejector framework to regression with abstention and introduced Bayes-consistent surrogates. Furthermore,
\citet{MohriAndorChoiCollinsMaoZhong2024learning} explored the
predictor-rejector framework from the perspective of learning with a fixed predictor,
applying their novel algorithms to decontextualization
tasks. \citet{li2024no} investigated the Bayes-consistency of
no-rejection learning in the setting of regression with
abstention. \citet{ramaswamy2018consistent} also studied the
confidence-based abstention in the multi-class classification, where
they show certain multi-class hinge loss formulations and a new
constructed polyhedral binary encoded predictions (BEP) surrogate loss
are Bayes-consistent. \citet{charoenphakdee2021classification}
proposed a cost-sensitive approach for the multi-class abstention,
where they decomposed the multi-class problem into multiple binary
cost-sensitive classification problems
\citep{elkan2001foundations}. They proposed a family of cost-sensitive
one-versus-all surrogate losses, which are Bayes-consistent in that
setting.

\citet{mozannar2020consistent} proposed instead for the multi-class
abstention setting a \emph{score-based formulation}, where, in
addition to the standard scoring functions associated to each label, a
new scoring function is associated to a new rejection label.
Rejection takes places when the score given to the rejection label is
higher than other scores and the rejector is therefore implicitly
defined via this specific rule. The authors suggested a surrogate loss
for their approach based on the cross-entropy (logistic loss with
softmax applied to neural networks outputs), which they proved to be
Bayes-consistent. More recently, \citet{caogeneralizing} gave a more
general family of Bayes-consistent surrogate losses for the score-based
formulation that can be built upon any consistent loss for the
standard multi-class classification problem. Most recent research
by \citet{pmlr-v206-mozannar23a} demonstrates that cross-entropy
score-based surrogate losses are not realizable $\sH$-consistent \citep{long2013consistency,zhang2020bayes}, in relation to
abstention loss. Instead, the authors propose a novel surrogate
loss that is proved to be realizable $\sH$-consistent when $\sH$ is
\emph{closed under scaling}, although its Bayes-consistency remains
unclear. The challenge of devising a surrogate loss that exhibits both
Bayes-consistency and realizable $\sH$-consistency remains an open
problem.

A problem directly related to our study is that of learning to defer,
which can be directly cast as an instance of learning with
abstention. There are several recent publications studying this
formulation of the problem
\citep{madras2018predict,raghu2019algorithmic,raghu2019direct,
  mozannar2020consistent,okati2021differentiable,wilder2021learning,
  bansal2021most,verma2022calibrated,narasimhanpost,verma2023learning,
  mao2023two,cao2023defense,MaoMohriZhong2024deferral,chen2024learning,mao2024regression}. \citet{raghu2019direct,
  wilder2021learning,bansal2021most} studied confidence-based methods
to make deferral decisions, which may be sub-optimal for low capital
models \citep{CortesDeSalvoMohri2016,CortesDeSalvoMohri2023}. To
overcome this limitation, \citet{mozannar2020consistent} proposed
cost-sensitive logistic loss and \citet{verma2022calibrated} proposed
cost-sensitive one-versus-all proper composite loss
\citep{reid2010composite}, both in the score-based
formulation. \citet{verma2023learning} further generalized the
surrogate loss in \citep{verma2022calibrated} to the setting of
deferring with multiple experts. Furthermore,
\citet{MaoMohriZhong2024deferral} introduced a new and more general
family of surrogate losses specifically tailored for this setting and
proved that these surrogate losses benefit from strong
$\sH$-consistency bounds. More recently, \citet{narasimhanpost}
pointed out that the existing surrogate losses for learning to defer
\citep{mozannar2020consistent,verma2022calibrated} may underfit in an
important practical setting and proposed a post-hoc correction for
these loss functions. Moreover, \citet{mao2023two} studied a two-stage
scenario for learning to defer with multiple experts, where a
predictor is first trained using a standard loss function such as
cross-entropy, and a deferral function is subsequently learned. They
introduced a novel family of surrogate loss functions and algorithms
for this crucial scenario, supported by $\sH$-consistency bounds. \citet{cao2023defense} introduced a new Bayes-consistent, asymmetric softmax-based surrogate loss, designed to yield valid estimates while avoiding the problem of unboundedness. Most recently, \citet{chen2024learning} incorporated deferral into a
sequential decision-making model, leading to improved theoretical
convergence and empirical performance. \citet{mao2024regression} proposed the framework of regression with multiple-expert deferral and novel surrogate losses that benefit from $\sH$-consistency bounds in that context.

\section{Discussion on experiments}
\label{app:experimemts}

This section presents a detailed analysis of the experimental results.

For CIFAR-10, the two-stage score-based abstention surrogate loss
outperforms the cross-entropy scored-based abstention surrogate loss
($\mu = 1.0$) used in \citep{mozannar2020consistent} by 1.26\%, and
outperforms the cross-entropy scored-based abstention surrogate loss
($\mu = 1.7$) used in \citep{caogeneralizing} by 0.4\%. Our results
for the score-based surrogate losses are also consistent with those of
\citet{caogeneralizing}, who showed that the scored-based abstention
loss \eqref{eq:sur-score} with $\ell_{\mu}$ adopted as the generalized
cross-entropy loss ($\mu=1.7$) performs better than the scored-based
abstention loss with $\ell_{\mu}$ adopted as the logistic loss
($\mu=1$). This agrees with our theoretical analysis based on
$\sH$-consistency bounds and minimizability gaps in
Theorem~\ref{Thm:bound_comp_sum} and
Theorem~\ref{Thm:gap-upper-bound-determi}, since both losses have the
same square-root functional form while the magnitude of the
minimizability gap decreases with $\mu$ in light of the fact that
$\sE_{\lsc_{\mu}}^*(\sH)$ is close for both losses.

Table~\ref{tab:comparison} also shows that on SVHN, using deeper
neural networks than \citep{caogeneralizing}, the cross-entropy
scored-based abstention loss ($\mu=1.7$) actually performs worse than
the cross-entropy scored-based abstention loss ($\mu=1$) in
\citep{mozannar2020consistent}, in contrast with the opposite results
observed in \citep{caogeneralizing} when using shallower neural
networks. This is consistent with our theoretical analysis based on
their $\sH$-consistency bounds (Theorem~\ref{Thm:bound_comp_sum}): the
minimizability gaps are basically the same while the dependency of the
multiplicative constant on the number of classes appears for
$\mu=1.7$, which makes the scored-based abstention loss
\eqref{eq:sur-score} with $\ell_{\mu}$ adopted as the generalized
cross-entropy loss ($\mu=1.7$) less favorable. Here too, the two-stage
score-based abstention surrogate loss is superior to both, with an
abstention loss 1.23\% lower than that of \citep{caogeneralizing} and
0.68\% lower than that of \citep{mozannar2020consistent}.

To further test the algorithms, we also carried out experiments on
CIFAR-100, with deeper neural networks. Table~\ref{tab:comparison}
shows that score-based abstention loss with generalized cross-entropy
adopted in \citep{caogeneralizing} does not perform well in this
case. In contrast, the score-based abstention loss with the logistic
loss adopted in \citep{mozannar2020consistent} performs better and
surpasses it by 4.59\%. Our two-stage score-based abstention loss is
still the most favorable, here too, with 0.86\% lower abstention loss
than that of \citep{mozannar2020consistent}. As with the case of SVHN,
the inferior performance of the cross-entropy scored-based abstention
surrogate loss ($\mu = 1.7$) can be seen from the dependency of the
multiplicative constant on the number of classes in $\sH$-consistency
bounds (Theorem~\ref{Thm:bound_comp_sum}), which is worse when the
number of classes is much larger as in the case of CIFAR-100.

\section{Proofs for score-based abstention losses}
\label{app:score}

To begin with the proof, we first introduce some notation. Recall
that we denote by $p(x, y) = \sD(Y = y \!\mid\! X = x)$ the
conditional probability of $Y=y$ given $X = x$. For simplicity of the
notation, we let $p(x,n+1)=1-c$ and denote by $ y_{\max}\in \sY \bigcup \curl*{n+1}$ the label associated to an input $x\in
\sX$, defined as $ y_{\max}=n+1$ if $1-c \geq \max_{y \in \sY} p(x,
y)$; otherwise, $ y_{\max}$ is defined as an element in $\sY$ with
the highest conditional probability, $ y_{\max} = \argmax_{y \in
  \sY} p(x, y)$, with the same deterministic strategy for breaking
ties as that of $ \hh(x)$.  Thus, the generalization error for a
score-based abstention surrogate loss can be rewritten as $
\sE_{\lsc}( h) = \mathbb{E}_{X} \bracket*{\sC_{\lsc}( h, x)} $,
where $\sC_{\lsc}( h,x)$ is the conditional $\lsc$-risk, defined by
\begin{align*}
\sC_{\lsc}( h,x) = \sum_{y\in \sY\bigcup \curl*{n+1}} p(x, y)  \ell( h,x, y).
\end{align*}
We denote by $\sC_{\lsc}^*( \sH,x) = \inf_{ h\in
  \sH}\sC_{\lsc}( h,x)$ the minimal conditional
$\lsc$-risk. Then, the minimizability gap can be rewritten as follows:
\begin{align*}
\sM_{\lsc}( \sH)
 = \sE^*_{\lsc}( \sH) - \mathbb{E}_{X} \bracket* {\sC_{\lsc}^*( \sH, x)}.
\end{align*}
We further refer to $\sC_{\lsc}( h,x)-\sC_{\lsc}^*( \sH,x)$ as
the calibration gap and denote it by $\Delta\sC_{\lsc, \sH}(
h,x)$.  We first prove a lemma on the calibration gap of the
score-based abstention loss. For any $x \in \sX$, we will denote by
$\mathsf H(x)$ the set of labels generated by hypotheses in $ \sH$:
$\mathsf H(x) = \curl*{ \hh(x) \colon  h \in  \sH}$.
\begin{restatable}{lemma}{CalibrationGapScore}
\label{lemma:calibration_gap_score}
For any $x \in \sX$,
the minimal conditional $\labsc$-risk and
the calibration gap for $\labsc$ can be expressed as follows:
\begin{align*}
\sC^*_{\labsc}( \sH,x) & = 1 - \max_{y\in \mathsf H(x)} p(x, y)\\
\Delta\sC_{\labsc, \sH}( h,x) & = \max_{y\in \mathsf H(x)} p(x, y) - p(x, \hh(x)).
\end{align*}
\end{restatable}
\begin{proof}
The conditional
$\labsc$-risk of $ h$ can be expressed as follows:
\begin{align*}
\sC_{\labsc}( h, x)
 = \sum_{y\in \sY}p(x, y)\1_{ \hh(x)\neq y}\1_{ \hh(x)\neq n + 1}+c \1_{ \hh(x) = n + 1}=1-p(x,  \hh(x)).
\end{align*}
Then, the minimal conditional $\labsc$-risk is given by
\[
\sC_{\labsc}^*(\sH,x) = 1 - \max_{y\in \mathsf H(x)} p(x, y),
\]
and the calibration gap can be expressed as follows:
\begin{align*}
  \Delta \sC_{\labsc,\sH}( h, x)
  = \sC_{\labsc}( h, x)-\sC_{\labsc}^*(\sH,x)= \max_{y\in \mathsf H(x)} p(x, y)-p(x,  \hh(x)).
\end{align*}
This completes the proof.
\end{proof}

Note that when $ \sH$ is symmetric, $\mathsf H(x)= \sY \bigcup \curl*{n+1}$. By
Lemma~\ref{lemma:calibration_gap_score}, in those cases, we obtain the
following result,
\begin{corollary}
\label{cor:calibration_gap_score}
Assume that $ \sH$ is symmetric. Then, for any $x \in \sX$,
the minimal conditional $\labsc$-risk and
the calibration gap for $\labsc$ can be expressed as follows:
\begin{align*}
\sC^*_{\labsc}( \sH,x) & =  1- p(x, y_{\max})\\
\Delta\sC_{\labsc, \sH}( h,x) & =  p(x, y_{\max}) - p(x, \hh(x)).
\end{align*}
\end{corollary}

\subsection{Proof of \texorpdfstring{$\sH$}{H}-Consistency bounds
  for Cross-Entropy Score-Based Surrogates (Theorem~\ref{Thm:bound_comp_sum})}
\label{app:bound_comp_sum}
\BoundCompSum*
\begin{proof}
The main proof idea is similar for each case of $\mu$: we will lower
bound the calibration gap of $\lsc_{\mu}$ by that of $\labsc$ by
carefully selecting a hypothesis $ h_{\lambda}$ in the hypothesis
set $ \sH$. In particular, we analyze different cases as follows.
\paragraph{The Case Where $\mu \in [0,1)$} 

For any $ h \in \sH$ and $x\in \sX$, choose hypothesis $
h_{\lambda} \in  \sH$ such that
\begin{align*}
 h_{\lambda}(x, y) = 
\begin{cases}
   h(x, y) & \text{if $y \not \in \curl*{ y_{\max},  \hh(x)}$}\\
  \log\paren*{\exp\bracket*{ h(x, y_{\max})} + \lambda} & \text{if $y =  \hh(x)$}\\
  \log\paren*{\exp\bracket*{ h(x, \hh(x))} -\lambda} & \text{if $y = y_{\max}$},
\end{cases} 
\end{align*}
where $\lambda = \frac{\exp\bracket*{ h(x, \hh(x))}p(x,
  \hh(x))^{\frac1{2-\mu}}-\exp\bracket*{ h(x, y_{\max})}p(x,
  y_{\max})^{\frac1{2-\mu}}}{p(x, y_{\max})^{\frac1{2-\mu}}+p(x,
  \hh(x))^{\frac1{2-\mu}}}$. The existence of such a $ h_{\lambda}$
in the hypothesis set $ \sH$ is guaranteed by the assumption that $ \sH$
is symmetry and complete. Thus, the calibration gap can be expressed
and lower-bounded as follows:
\begin{align*}
&(1-\mu)\Delta \sC_{\lsc_{\mu}, \sH}( h, x)\\
& = (1-\mu)\paren*{\sC_{\lsc_{\mu}}( h, x) - \sC^*_{\lsc_{\mu}}( \sH, x)}\\
& \geq (1-\mu)\paren*{\sC_{\lsc_{\mu}}( h, x) - \sC_{\lsc_{\mu}}( h_{\lambda}, x)}\\
& = p(x, y_{\max}) \paren*{\bracket*{\sum_{y'\in  \sY\bigcup \curl*{n+1}} e^{ h(x, y')- h(x, y_{\max})}}^{1 - \mu }-1} +p(x, \hh(x)) \paren*{\bracket*{\sum_{y'\in  \sY\bigcup \curl*{n+1}} e^{ h(x, y')- h(x, \hh(x))}}^{1 - \mu }-1}\\
&  -p(x, y_{\max}) \paren*{\bracket*{\sum_{y'\in  \sY\bigcup \curl*{n+1}}e^{ h(x, y')- h(x, \hh(x))+\lambda}}^{1 - \mu }-1} -p(x, \hh(x)) \paren*{\bracket*{\sum_{y'\in  \sY\bigcup \curl*{n+1}}e^{ h(x, y')- h(x, y_{\max})-\lambda}}^{1 - \mu }-1}\\
& = p(x, y_{\max})\bracket*{\sum_{y'\in  \sY\bigcup \curl*{n+1}} e^{ h(x, y')-e^{ h(x, y_{\max})}}}^{1 - \mu } \\
& \qquad - p(x, y_{\max})\bracket*{\frac{\sum_{y'\in  \sY\bigcup \curl*{n+1}}e^{ h(x, y')}\bracket*{p(x, y_{\max})^{\frac1{2 - \mu }}+p(x, \hh(x))^{\frac1{2 - \mu }}}}{\bracket*{e^{ h(x, y_{\max})} + e^{ h(x, \hh(x))}}p(x, y_{\max})^{\frac1{2 - \mu }}}}^{1 - \mu }\\
& \quad +p(x, \hh(x))\bracket*{\sum_{y'\in  \sY\bigcup \curl*{n+1}} e^{ h(x, y')- h(x, \hh(x))}}^{1 - \mu }\\
& \qquad - p(x, \hh(x))\bracket*{\frac{\sum_{y'\in  \sY\bigcup \curl*{n+1}}e^{ h(x, y')}\bracket*{p(x, y_{\max})^{\frac1{2 - \mu }}+p(x, \hh(x))^{\frac1{2 - \mu }}}}{\bracket*{e^{ h(x, y_{\max})} + e^{ h(x, \hh(x))}}p(x, \hh(x))^{\frac1{2 - \mu }}}}^{1 - \mu }\\
&\geq p(x, y_{\max})\bracket*{e^{ h(x, \hh(x))- h(x, y_{\max})}+1}^{1 - \mu } - p(x, y_{\max})\bracket*{\frac{p(x, y_{\max})^{\frac1{2 - \mu }}+p(x, \hh(x))^{\frac1{2 - \mu }}}{p(x, y_{\max})^{\frac1{2 - \mu }}}}^{1 - \mu }\\
& \quad +p(x, \hh(x))\bracket*{e^{ h(x, y_{\max})- h(x, \hh(x))}+1}^{1 - \mu } - p(x, \hh(x))\bracket*{\frac{p(x, y_{\max})^{\frac1{2 - \mu }}+p(x, \hh(x))^{\frac1{2 - \mu }}}{p(x, \hh(x))^{\frac1{2 - \mu }}}}^{1 - \mu }
\tag{$\sum_{y'\in  \sY\bigcup \curl*{n+1}} e^{ h(x, y')}\geq e^{ h(x, \hh(x))}+e^{ h(x, y_{\max})}$}\\
&\geq p(x, y_{\max})2^{1 - \mu } - p(x, y_{\max})^{\frac1{2 - \mu }} \bracket*{p(x, y_{\max})^{\frac1{2 - \mu }}+p(x, \hh(x))^{\frac1{2 - \mu }}}^{1 - \mu }\\
&\quad +p(x, \hh(x))2^{1 - \mu } - p(x, \hh(x))^{\frac1{2 - \mu }}\bracket*{p(x, y_{\max})^{\frac1{2 - \mu }}+p(x, \hh(x))^{\frac1{2 - \mu }}}^{1 - \mu }
\tag{minimum is attained when $e^{ h(x, \hh(x))}=e^{ h(x, y_{\max})}$}\\
& = 2^{1-\mu}\paren*{p(x, y_{\max})+p(x, \hh(x))}-\bracket*{p(x, y_{\max})^{\frac1{2 - \mu }}+p(x, \hh(x))^{\frac1{2 - \mu }}}^{2 - \mu }\\
& = 2^{2-\mu}\bracket*{\paren*{\frac{p(x, y_{\max})+p(x, \hh(x))}{2}}-\bracket*{\frac{p(x, y_{\max})^{\frac1{2 - \mu }}+p(x, \hh(x))^{\frac1{2 - \mu }}}{2}}^{2 - \mu }}\\
& \geq \frac{1-\mu}{(2-c)2^{\mu}(2-\mu)}\paren*{p(x, y_{\max}) - p(x, \hh(x))}^2
\tag{$p(x, y_{\max})+p(x, \hh(x))\leq 2-c$ and by analyzing the Taylor expansion}\\
& = \frac{1-\mu}{(2-c)2^{\mu}(2-\mu)}\Delta\sC_{\labsc, \sH}( h,x)^2 \tag{Corollary~\ref{cor:calibration_gap_score}}
\end{align*}
Thus, we have
\begin{align*}
\sE_{\labsc}( h) - \sE_{\labsc}^*( \sH) + \sM_{\labsc}( \sH)
& = \E_{X}\bracket*{\Delta \sC_{\labsc,\sH}( h, x)}\\
& \leq \E_X\bracket*{\Gamma_{\mu}\paren*{\Delta \sC_{\lsc_{\mu},\sH}( h, x)}}\\
& \leq \Gamma_{\mu}\paren*{\E_X\bracket*{\Delta \sC_{\lsc_{\mu},\sH}( h, x)}}
\tag{$\Gamma_{\mu}$ is concave}\\
& = \Gamma_{\mu}\paren*{\sE_{\lsc_{\mu}}( h)-\sE_{\lsc_{\mu}}^*( \sH) +\sM_{\lsc_{\mu}}( \sH)},
\end{align*}
where $\Gamma_{\mu}(t)=\sqrt{(2-c)2^{\mu}(2-\mu) t}$.

\paragraph{The Case Where $\mu =1$} 

For any $ h \in \sH$ and $x\in \sX$, choose hypothesis $ h_{\lambda} \in  \sH$ such that
\begin{align*}
 h_{\lambda}(x, y) = 
\begin{cases}
   h(x, y) & \text{if $y \not \in \curl*{ y_{\max},  \hh(x)}$}\\
  \log\paren*{\exp\bracket*{ h(x, y_{\max})} + \lambda} & \text{if $y =  \hh(x)$}\\
  \log\paren*{\exp\bracket*{ h(x, \hh(x))} -\lambda} & \text{if $y = y_{\max}$}
\end{cases} 
\end{align*}
where $\lambda = \frac{\exp\bracket*{ h(x, \hh(x))}p(x,
  \hh(x))-\exp\bracket*{ h(x,  y_{\max})}p(x,
  y_{\max})}{p(x, y_{\max})+p(x, \hh(x))}$. The existence of
such a $ h_{\lambda}$ in hypothesis set $ \sH$ is guaranteed by
the fact that $ \sH$ is symmetry and complete. Thus, the
calibration gap can be expressed and lower-bounded as follows:
\begin{align*}
& \Delta \sC_{\lsc_{\mu}, \sH}( h, x)\\
& = \sC_{\lsc_{\mu}}( h, x) - \sC^*_{\lsc_{\mu}}( \sH, x)\\
& \geq \sC_{\lsc_{\mu}}( h, x) - \sC_{\lsc_{\mu}}( h_{\lambda}, x)\\
& =-p(x, y_{\max}) \log\bracket*{e^{ h(x, y_{\max})}}-p(x, \hh(x)) \log\bracket*{e^{ h(x, \hh(x))}}\\
& \qquad +p(x, y_{\max})\log\bracket*{ e^{ h(x, \hh(x))}-\lambda}+p(x, \hh(x))\log\bracket*{e^{ h(x, y_{\max})}+\lambda}\\
& = p(x, y_{\max})\log\bracket*{\frac{\bracket*{e^{ h(x, y_{\max})} + e^{ h(x, \hh(x))}}p(x, y_{\max})}{e^{ h(x, y_{\max})}\bracket*{p(x, y_{\max})+p(x, \hh(x))}}}+p(x, \hh(x))\log\bracket*{\frac{\bracket*{e^{ h(x, y_{\max})} + e^{ h(x, \hh(x))}}p(x, \hh(x))}{ e^{ h(x, \hh(x))}\bracket*{p(x, y_{\max})+p(x, \hh(x))}}}\\
&\geq p(x, y_{\max})\log\bracket*{\frac{2p(x, y_{\max})}{p(x, y_{\max})+p(x, \hh(x))}} + p(x, \hh(x))\log\bracket*{\frac{2p(x, \hh(x))}{p(x, y_{\max})+p(x, \hh(x))}}
\tag{minimum is attained when $e^{ h(x, \hh(x))}=e^{ h(x, y_{\max})}$}\\
& \geq \bracket*{p(x, y_{\max})+p(x, \hh(x))} \times \frac12\bracket*{ \abs*{\frac{p(x, y_{\max})}{p(x, y_{\max})+p(x, \hh(x))}-\frac12}+\abs*{\frac{p(x, \hh(x))}{p(x, y_{\max})+p(x, \hh(x))}-\frac12}}^2
\tag{Pinsker’s inequality \citep[Proposition~E.7]{MohriRostamizadehTalwalkar2018}}\\
& = \bracket*{p(x, y_{\max})+p(x, \hh(x))} \times \frac12 \bracket*{\frac{p(x, y_{\max})-p(x, \hh(x))}{p(x, y_{\max})+p(x, \hh(x))}}^2
\tag{$p(x, y_{\max})\geq p(x, \hh(x))$}\\\\
& \geq \frac1{2(2-c)} \paren*{p(x, y_{\max}) - p(x, \hh(x))}^2\\
\tag{$p(x, y_{\max})+p(x, \hh(x))\leq 2-c$}\\
& = \frac1{2(2-c)}\Delta\sC_{\labsc, \sH}( h,x)^2 \tag{Corollary~\ref{cor:calibration_gap_score}}
\end{align*}
Thus, we have
\begin{align*}
\sE_{\labsc}( h) - \sE_{\labsc}^*( \sH) + \sM_{\labsc}( \sH)
& = \E_{X}\bracket*{\Delta \sC_{\labsc,\sH}( h, x)}\\
& \leq \E_X\bracket*{\Gamma_{\mu}\paren*{\Delta \sC_{\lsc_{\mu},\sH}( h, x)}}\\
& \leq \Gamma_{\mu}\paren*{\E_X\bracket*{\Delta \sC_{\lsc_{\mu},\sH}( h, x)}}
\tag{$\Gamma_{\mu}$ is concave}\\
& = \Gamma_{\mu}\paren*{\sE_{\lsc_{\mu}}( h)-\sE_{\lsc_{\mu}}^*( \sH) +\sM_{\lsc_{\mu}}( \sH)},
\end{align*}
where $\Gamma_{\mu}(t)=\sqrt{2(2-c)t }$.

\paragraph{The Case Where $\mu \in [2,\plus \infty)$} 

For any $ h \in \sH$ and $x\in \sX$, choose hypothesis $ h_{\lambda} \in  \sH$ such that
\begin{align*}
 h_{\lambda}(x, y) = 
\begin{cases}
   h(x, y) & \text{if $y \not \in \curl*{ y_{\max},  \hh(x)}$}\\
  \log\paren*{\exp\bracket*{ h(x, y_{\max})} + \lambda} & \text{if $y =  \hh(x)$}\\
  \log\paren*{\exp\bracket*{ h(x, \hh(x))} -\lambda} & \text{if $y = y_{\max}$}
\end{cases} 
\end{align*}
where $\lambda = -\exp\bracket*{ h(x, y_{\max})}$. The existence of
such a $ h_{\lambda}$ in hypothesis set $ \sH$ is guaranteed by
the fact that $ \sH$ is symmetry and complete. Thus, the
calibration gap can be expressed and lower-bounded as follows:
\begin{align*}
& (\mu-1)\Delta \sC_{\lsc_{\mu}, \sH}( h, x)\\
& = (\mu-1) \paren*{ \sC_{\lsc_{\mu}}( h, x) - \sC^*_{\lsc_{\mu}}( \sH, x)}\\
& \geq (\mu-1)\paren*{\sC_{\lsc_{\mu}}( h, x) - \sC_{\lsc_{\mu}}( h_{\lambda}, x)}\\
& =p(x, y_{\max}) \paren*{1-\bracket*{\frac{e^{ h(x, y_{\max})}}{\sum_{y'\in  \sY\bigcup \curl*{n+1}}e^{ h(x, y')}}}^{\mu-1}} +p(x, \hh(x)) \paren*{1-\bracket*{\frac{e^{ h(x, \hh(x))}}{\sum_{y'\in  \sY\bigcup \curl*{n+1}}e^{ h(x, y')}}}^{\mu-1}}\\
&  -p(x, y_{\max}) \paren*{1-\bracket*{\frac{e^{ h(x, \hh(x))}-\mu}{\sum_{y'\in  \sY\bigcup \curl*{n+1}}e^{ h(x, y')}}}^{\mu-1}} -p(x, \hh(x)) \paren*{1-\bracket*{\frac{e^{ h(x, y_{\max})}+\mu}{\sum_{y'\in  \sY\bigcup \curl*{n+1}}e^{ h(x, y')}}}^{\mu-1}}\\
& = p(x, y_{\max})\bracket*{\frac{e^{ h(x, \hh(x))}+e^{ h(x, y_{\max})}}{\sum_{y'\in  \sY\bigcup \curl*{n+1}}e^{h(x, y')}}}^{\mu-1}-p(x, y_{\max})\bracket*{\frac{e^{ h(x, y_{\max})}}{\sum_{y'\in  \sY\bigcup \curl*{n+1}}e^{ h(x, y')}}}^{\mu-1}\\
& \quad -p(x, \hh(x))\bracket*{\frac{e^{ h(x, \hh(x))}}{\sum_{y'\in  \sY\bigcup \curl*{n+1}}e^{ h(x, y')}}}^{\mu-1}\\
&\geq p(x, y_{\max})\bracket*{\frac{e^{ h(x, \hh(x))}}{\sum_{y'\in  \sY\bigcup \curl*{n+1}}e^{ h(x, y')}}}^{\mu-1}-p(x, \hh(x))\bracket*{\frac{e^{ h(x, \hh(x))}}{\sum_{y'\in  \sY\bigcup \curl*{n+1}}e^{ h(x, y')}}}^{\mu-1}
\tag{$(x+y)^{\mu-1}\geq x^{\mu -1 } + y^{\mu-1}$, $\forall\, x, y\geq 0$, $\mu\geq 2$}\\
&\geq \frac{1}{(n+1)^{\mu-1}}
\paren*{p(x, y_{\max}) - p(x, \hh(x))} \tag{$\frac{e^{ h(x, \hh(x))}}{\sum_{y'\in  \sY\bigcup \curl*{n+1}}e^{h(x, y')}}\geq \frac1{n+1}$}\\
& = \frac{1}{(n+1)^{\mu-1}}\Delta\sC_{\labsc, \sH}( h,x) \tag{Corollary~\ref{cor:calibration_gap_score}}
\end{align*}
Thus, we have
\begin{align*}
\sE_{\labsc}( h) - \sE_{\labsc}^*( \sH) + \sM_{\labsc}( \sH)
& = \E_{X}\bracket*{\Delta \sC_{\labsc,\sH}( h, x)}\\
& \leq \E_X\bracket*{\Gamma_{\mu}\paren*{\Delta \sC_{\lsc_{\mu},\sH}( h, x)}}\\
& \leq \Gamma_{\mu}\paren*{\E_X\bracket*{\Delta \sC_{\lsc_{\mu},\sH}( h, x)}}
\tag{$\Gamma_{\mu}$ is concave}\\
& = \Gamma_{\mu}\paren*{\sE_{\lsc_{\mu}}( h)-\sE_{\lsc_{\mu}}^*( \sH) +\sM_{\lsc_{\mu}}( \sH)},
\end{align*}
where $\Gamma_{\mu}(t)=(\mu - 1)(n+1)^{\mu - 1} t$.

\paragraph{The Case Where $\mu \in (1,2)$}

For any $ h \in \sH$ and $x\in \sX$, choose hypothesis $ h_{\lambda} \in  \sH$ such that
\begin{align*}
 h_{\lambda}(x, y) = 
\begin{cases}
   h(x, y) & \text{if $y \not \in \curl*{ y_{\max},  \hh(x)}$}\\
  \log\paren*{\exp\bracket*{ h(x, y_{\max})} + \lambda} & \text{if $y =  \hh(x)$}\\
  \log\paren*{\exp\bracket*{ h(x, \hh(x))} -\lambda} & \text{if $y = y_{\max}$}
\end{cases} 
\end{align*}
where $\lambda = \frac{\exp\bracket*{ h(x, \hh(x))}p(x,
  y_{\max})^{\frac1{\mu-2}}-\exp\bracket*{ h(x,
    y_{\max})}p(x, \hh(x))^{\frac1{\mu-2}}}{p(x,
  y_{\max})^{\frac1{\mu-2}}+p(x, \hh(x))^{\frac1{\mu-2}}}$. The
existence of such a $ h_{\lambda}$ in hypothesis set $ \sH$ is
guaranteed by the fact that $ \sH$ is symmetry and complete. Thus,
the calibration gap can be lower-bounded as follows:
\begin{align*}
& (\mu-1)\Delta \sC_{\lsc_{\mu}, \sH}( h, x)\\
& \geq (\mu-1)\paren*{\sC_{\lsc_{\mu}}( h, x) - \sC_{\lsc_{\mu}}( h_{\lambda}, x)}\\
& = p(x, y_{\max}) \paren*{1-\bracket*{\sum_{y'\in  \sY\bigcup \curl*{n+1}} e^{ h(x, y')- h(x, y_{\max})}}^{1 - \mu }} +p(x, \hh(x)) \paren*{1-\bracket*{\sum_{y'\in  \sY\bigcup \curl*{n+1}} e^{ h(x, y')- h(x, \hh(x))}}^{1 - \mu }}\\
&  -p(x, y_{\max}) \paren*{1-\bracket*{\sum_{y'\in  \sY\bigcup \curl*{n+1}}e^{ h(x, y')- h(x, \hh(x))+\lambda}}^{1 - \mu }} -p(x, \hh(x)) \paren*{1-\bracket*{\sum_{y'\in  \sY\bigcup \curl*{n+1}}e^{ h(x, y')- h(x, y_{\max})-\lambda}}^{1 - \mu }}\\
& = -p(x, y_{\max})\bracket*{\sum_{y'\in  \sY\bigcup \curl*{n+1}} e^{ h(x, y')-e^{ h(x, y_{\max})}}}^{1 - \mu } - p(x, \hh(x))\bracket*{\sum_{y'\in  \sY\bigcup \curl*{n+1}} e^{ h(x, y')- h(x, \hh(x))}}^{1 - \mu }\\
& \quad + p(x, y_{\max})\bracket*{\frac{\sum_{y'\in  \sY\bigcup \curl*{n+1}}e^{ h(x, y')}\bracket*{p(x, y_{\max})^{\frac1{\mu-2}}+p(x, \hh(x))^{\frac1{\mu-2 }}}}{\bracket*{e^{ h(x, y_{\max})} + e^{ h(x, \hh(x))}}p(x, \hh(x))^{\frac1{\mu-2}}}}^{1 - \mu }\\
& \qquad + p(x, \hh(x))\bracket*{\frac{\sum_{y'\in  \sY\bigcup \curl*{n+1}}e^{ h(x, y')}\bracket*{p(x, y_{\max})^{\frac1{\mu-2 }}+p(x, \hh(x))^{\frac1{\mu-2}}}}{\bracket*{e^{ h(x, y_{\max})} + e^{ h(x, \hh(x))}}p(x, y_{\max})^{\frac1{\mu-2}}}}^{1 - \mu }\\
&\geq \frac{1}{(n+1)^{\mu-1}}\paren*{p(x, y_{\max})\bracket*{\frac{\bracket*{e^{h(x, y_{\max})} + e^{h(x, \hh(x))}}p(x, \hh(x))^{\frac1{\mu-2}}}{e^{h(x, \hh(x))}\bracket*{p(x, y_{\max})^{\frac1{\mu-2}}+p(x, \hh(x))^{\frac1{\mu-2}}}}}^{\mu-1}-p(x, y_{\max})\bracket*{e^{h(x, y_{\max})-h(x, \hh(x))}}^{\mu-1}}\\
&\quad +\frac{1}{(n+1)^{\mu-1}}\paren*{p(x, \hh(x))\bracket*{\frac{\bracket*{e^{h(x, y_{\max})} + e^{h(x, \hh(x))}}p(x, y_{\max})^{\frac1{\mu-2}}}{e^{h(x, \hh(x))}\bracket*{p(x, y_{\max})^{\frac1{\mu-2}}+p(x, \hh(x))^{\frac1{\mu-2}}}}}^{\mu-1}-p(x, \hh(x))}
\tag{$\frac{e^{ h(x, \hh(x))}}{\sum_{y'\in  \sY\bigcup \curl*{n+1}} e^{ h(x, y')}}\geq \frac{1}{(n+1)^{\mu-1}}$}\\
&\geq \frac{1}{(n+1)^{\mu-1}}\paren*{p(x, y_{\max})\bracket*{\frac{2p(x, \hh(x))^{\frac1{\mu-2}}}{p(x, y_{\max})^{\frac1{\mu-2}}+p(x, \hh(x))^{\frac1{\mu-2}}}}^{\mu-1}-p(x, y_{\max})}\\
& \quad + \frac{1}{(n+1)^{\mu-1}}\paren*{p(x, \hh(x))\bracket*{\frac{2p(x, y_{\max})^{\frac1{\mu-2}}}{p(x, y_{\max})^{\frac1{\mu-2}}+p(x, \hh(x))^{\frac1{\mu-2}}}}^{\mu-1}-p(x, \hh(x))}
\tag{minimum is attained when $e^{ h(x, \hh(x))}=e^{ h(x, y_{\max})}$}\\
& = \frac{1}{(n+1)^{\mu-1}}\paren*{2^{\mu-1}\bracket*{p(x, y_{\max})^{\frac1{2-\mu}}+p(x, \hh(x))^{\frac1{2-\mu}}}^{2-\mu}-p(x, y_{\max})-p(x, \hh(x))}\\
& = \frac{2}{(n+1)^{\mu-1}}\paren*{\bracket*{\frac{p(x, y_{\max})^{\frac1{2-\mu}}+p(x, \hh(x))^{\frac1{2-\mu}}}{2}}^{2-\mu}-\frac{p(x, y_{\max})+p(x, \hh(x))}{2}}\\
& \geq \frac{\mu-1}{2(2-c)(n+1)^{\mu-1}}\paren*{p(x, y_{\max}) - p(x, \hh(x))}^2
\tag{$p(x, y_{\max})+p(x, \hh(x))\leq 2-c$ and by analyzing the Taylor expansion}\\
& = \frac{\mu-1}{2(2-c)(n+1)^{\mu-1}}\Delta\sC_{\labsc, \sH}( h,x)^2 \tag{Corollary~\ref{cor:calibration_gap_score}}
\end{align*}
Thus, we have
\begin{align*}
\sE_{\labsc}( h) - \sE_{\labsc}^*( \sH) + \sM_{\labsc}( \sH)
& = \E_{X}\bracket*{\Delta \sC_{\labsc,\sH}( h, x)}\\
& \leq \E_X\bracket*{\Gamma_{\mu}\paren*{\Delta \sC_{\lsc_{\mu},\sH}( h, x)}}\\
& \leq \Gamma_{\mu}\paren*{\E_X\bracket*{\Delta \sC_{\lsc_{\mu},\sH}( h, x)}}
\tag{$\Gamma_{\mu}$ is concave}\\
& = \Gamma_{\mu}\paren*{\sE_{\lsc_{\mu}}( h)-\sE_{\lsc_{\mu}}^*( \sH) +\sM_{\lsc_{\mu}}( \sH)},
\end{align*}
where $\Gamma_{\mu}(t)=\sqrt{2(2-c)(n+1)^{\mu-1}t }$.
\end{proof}

\subsection{Characterization of Minimizability
  Gaps (Theorem~\ref{Thm:gap-upper-bound-determi})}
\label{app:gap-upper-bound-determi}
\GapUpperBoundDetermi*
\begin{proof}
Let $s_{ h}(x, y)=\frac{e^{ h(x, y)}}{\sum_{y'\in  \sY\bigcup \curl*{n+1}} h(x, y')}\in [0,1]$, $\forall y\in  \sY$.
By the definition, for any deterministic distribution, $\sM_{\lsc_{\mu}}( \sH)
 = \sE^*_{\lsc_{\mu}}( \sH) - \mathbb{E}_{X} \bracket* {\inf_{ h \in  \sH}\sC_{\lsc_{\mu}}( \sH, x)}$, where 
\begin{align*}
&\sC_{\lsc_{\mu}}( h,x)\\
& = \sum_{y\in \sY\bigcup \curl*{n+1}} p(x, y)  \ell_{\mu}( h,x, y)\\
& =  \ell_{\mu}( h,x, y_{\max}) + (1-c) \ell_{\mu}( h,x,n+1)\\
& =\begin{cases}
\frac{1}{1 - \mu} \paren*{\bracket*{\sum_{y'\in \sY\bigcup \curl*{n+1}} e^{{ h(x, y') -  h(x, y_{\max})}}}^{1 - \mu} - 1} + (1-c) \frac{1}{1 - \mu} \paren*{\bracket*{\sum_{y'\in \sY\bigcup \curl*{n+1}} e^{{ h(x, y') -  h(x, n+1)}}}^{1 - \mu} - 1} & \mu\neq 1  \\
\log\paren*{\sum_{y'\in  \sY\bigcup \curl*{n+1}} e^{ h(x, y') -  h(x, y_{\max})}} + (1-c) \log\paren*{\sum_{y'\in  \sY\bigcup \curl*{n+1}} e^{ h(x, y') -  h(x, n+1)}} & \mu = 1.
\end{cases}\\
& = \begin{cases}
\frac{1}{1 - \mu} \paren*{s_{ h}\paren*{x, y_{\max}}^{\mu-1} - 1} + (1-c) \frac{1}{1 - \mu} \paren*{\bracket*{s_{ h}\paren*{x,n+1}}^{\mu-1} - 1} & \mu\neq 1  \\
-\log\paren*{s_{ h}\paren*{x, y_{\max}}} - (1-c) \log\paren*{s_{ h}\paren*{x,n+1}} & \mu = 1.
\end{cases}
\end{align*}
Since $0\leq s_{ h}(x, y_{\max})+s_{ h}(x,n+1)\leq 1$, by taking the partial derivative, we obtain that the minimum can be attained by
\begin{align}
\label{eq:min}
\begin{cases}
  s^*_{ h}(x, y_{\max})=\frac{1}{1+(1-c)^{\frac{1}{2-\mu}}} \text{ and } s^*_{ h}(x,n+1)
  =\frac{(1-c)^{\frac{1}{2-\mu}}}{1+(1-c)^{\frac{1}{2-\mu}}} & \mu \neq 2\\
s^*_{ h}(x, y_{\max}) = 1 \text{ and } s^*_{ h}(x,n+1)=0 & \mu =2.
\end{cases}
\end{align}
Since $ \sH$ is symmetric and complete, there exists $ h \in 
\sH$ such that \eqref{eq:min} is achieved. Therefore,
\begin{align*}
\inf_{ h \in  \sH}\sC_{\lsc_{\mu}}( \sH, x) 
& =
\begin{cases}
\frac{1}{1 - \mu} \paren*{s^*_{ h}\paren*{x, y_{\max}}^{\mu-1} - 1} + (1-c) \frac{1}{1 - \mu} \paren*{\bracket*{s^*_{ h}\paren*{x,n+1}}^{\mu-1} - 1} & \mu\neq 1  \\
-\log\paren*{s^*_{ h}\paren*{x, y_{\max}}} - (1-c) \log\paren*{s^*_{ h}\paren*{x,n+1}} & \mu = 1
\end{cases}\\
& =
 \begin{cases}
   \frac{1}{1 - \mu} \bracket*{\bracket*{1+\paren*{1-c}^{\frac{1}{2-\mu}}}^{2 - \mu}
     \mspace{-20mu} - (2-c)} & \mu \notin \curl*{1,2}\\
-\log \paren*{\frac{1}{2-c}}-(1-c)\log \paren*{\frac{1-c}{2-c}}
 & \mu=1\\
1-c & \mu =2.
\end{cases}
\end{align*}
Since $\inf_{ h \in  \sH}\sC_{\lsc_{\mu}}( \sH, x) $ is
independent of $x$, we obtain that $\mathbb{E}_{X} \bracket*
{\inf_{ h \in  \sH}\sC_{\lsc_{\mu}}( \sH, x)}=\inf_{ h \in
   \sH}\sC_{\lsc_{\mu}}( \sH, x)$,
which completes the proof.
\end{proof}

\subsection{Proof of General Transformation of \texorpdfstring{$\sH$}{H}-Consistency Bounds (Theorem~\ref{Thm:bound-score})}
\label{app:bound-score}
\BoundScore*
\begin{proof}
By Lemma~\ref{lemma:calibration_gap_score}, the calibration gap of
$\labsc$ can be expressed and upper-bounded as follows:
\begin{align*}
& \Delta \sC_{\labsc,\sH}( h, x)\\
& = \sC_{\labsc}( h, x)-\sC_{\labsc}^*(\sH,x)\\
& = \max_{y\in \mathsf H(x)} p(x, y) - p(x, \hh(x))\\
& = (2 - c) \paren*{\max_{y\in \mathsf H(x)} \ov p(x, y) -  \ov p(x,  \hh(x))}\tag{Let $ \ov p(x, y) = \frac{p(x, y)}{2 - c}\1_{y\in \sY}+\frac{1 - c}{2 - c}\1_{y = n + 1}$}\\
& =(2 - c) \Delta \sC_{\ell_{0-1},\sH}( h, x)\tag{By \citep[Lemma~3]{awasthi2022multi}}\\
& \leq (2 - c)\Gamma\paren*{\Delta \sC_{ \ell, \sH}( h, x)} \tag{By $ \sH$-consistency
bound of $ \ell$}\\
& = (2 - c)\Gamma\paren*{\sum_{y\in \sY \bigcup \curl*{n+1}} \ov p(x, y) \ell( h, x, y)-\inf_{ h \in \sH}\sum_{y\in \sY \bigcup \curl*{n+1}} \ov p(x, y) \ell( h, x, y)} \\
& = (2 - c)\Gamma\paren*{\sum_{y\in \sY}\frac{p(x, y)}{2 - c} \ell( h, x, y)+\frac{1 - c}{2 - c} \ell( h, x, n + 1)-\inf_{ h \in \sH}\paren*{\sum_{y\in \sY}\frac{p(x, y)}{2 - c} \ell( h, x, y)+\frac{1 - c}{2 - c} \ell( h, x, n + 1)}}\tag{Plug in $ \ov p(x, y) = \frac{p(x, y)}{2 - c}\1_{y\in \sY}+\frac{1 - c}{2 - c}\1_{y = n + 1}$} \\
& = (2 - c)\Gamma\paren*{\frac{1}{2 - c}\bracket*{\sum_{y\in \sY}p(x, y)\lsc( h, x, y)-\inf_{ h \in \sH}\sum_{y\in \sY}p(x, y)\lsc( h, x, y)}}\\
& = (2 - c)\Gamma\paren*{\frac{1}{2 - c}\Delta \sC_{\lsc,\sH}( h, x)}.
\end{align*}
Thus, we have
\begin{align*}
\sE_{\labsc}( h) - \sE_{\labsc}^*( \sH) + \sM_{\labsc}( \sH)
& = \E_{X}\bracket*{\Delta \sC_{\labsc,\sH}( h, x)}\\
& \leq \E_X\bracket*{(2 - c)\Gamma\paren*{\frac{1}{2 - c}\Delta \sC_{\lsc,\sH}( h, x)}}\\
& \leq (2 - c) \Gamma\paren*{\frac{1}{2 - c} \E_X\bracket*{\Delta \sC_{\lsc,\sH}( h, x)}}
\tag{$\Gamma$ is concave}\\
& = (2 - c) \Gamma\paren*{\frac{\sE_{\lsc}( h)-\sE_{\lsc}^*( \sH) +\sM_{\lsc}( \sH)}{2 - c}},
\end{align*}
which completes the proof.
\end{proof}

\subsection{Proof of \texorpdfstring{$\sH$}{H}-Consistency Bounds for Two-Stage Surrogates (Theorem~\ref{Thm:bound-general-two-step})}
\label{app:bound-general-two-step}
\BoundGenralTwoStep*
\begin{proof}
For any $h=(h_{\sY},h_{n+1})$,
we can rewrite $\sE_{\labs}(h)-\sE^*_{\labs}\paren*{\sH}+\sM_{\labs}(\sH)$ as 
\begin{equation}
\label{eq:expression-two-step}
\begin{aligned}
& \sE_{\labs}(h)-\sE^*_{\labs}\paren*{\sH}+\sM_{\labs}(\sH)\\
& =  \E_{X}\bracket*{\sC_{\labs}(h,x)-\sC^*_{\labs}(\sH,x)} \\
& =  \E_{X}\bracket*{\sC_{\labs}(h,x)-\inf_{h_{n+1}\in \sH_{n+1}}\sC_{\labs}(h,x)+\inf_{h_{n+1}\in \sH_{n+1}}\sC_{\labs}(h,x)-\sC^*_{\labs}(\sH, x)}\\
& = \E_{X}\bracket*{\sC_{\labs}(h,x)-\inf_{h_{n+1}\in \sH_{n+1}}\sC_{\labs}(h,x)}+\E_{X}\bracket*{\inf_{h_{n+1}\in \sH_{n+1}}\sC_{\labs}(h,x)-\sC^*_{\labs}(\sH,x)}
\end{aligned}
\end{equation}
By the assumptions, we have
\begin{align*}
& \sC_{\labs}(h,x)-\inf_{h_{n+1}\in \sH_{n+1}}\sC_{\labs}(h, x)\\
& = \sum_{y\in \sY}p(x, y)\1_{ \hh_{\sY}(x)\neq y}\1_{ \hh(x)\neq n + 1} + c \1_{ \hh(x) = n + 1}-\inf_{h_{n+1}\in \sH_{n+1}}\paren*{\sum_{y\in \sY}p(x, y)\1_{ \hh_{\sY}(x)\neq y}\1_{ \hh(x)\neq n + 1} + c \1_{ \hh(x) = n + 1}}\\
& = \paren*{\sum_{y\in \sY}p(x, y)\1_{\hh_{\sY}(x)\neq y} + c}\times \bigg[\eta(x)\ell_{0-1}^{\rm{binary}}\paren*{h_{n+1}-\max_{y\in \sY}h_{\sY}(x, y),x,+1}\\
& \quad +(1-\eta(x))\ell_{0-1}^{\rm{binary}}\paren*{h_{n+1}-\max_{y\in \sY}h_{\sY}(x, y),x,-1}\\
& \qquad -\inf_{h_{n+1}\in \sH_{n+1}}\paren*{\eta(x)\ell_{0-1}^{\rm{binary}}\paren*{h_{n+1}-\max_{y\in \sY}h_{\sY}(x, y),x,+1}+(1-\eta(x))\ell_{0-1}^{\rm{binary}}\paren*{h_{n+1}-\max_{y\in \sY}h_{\sY}(x, y),x,-1}}\bigg]\tag{Let $\eta(x) = \frac{\sum_{y\in \sY}p(x, y)\1_{\hh_{\sY}(x)\neq y}}{\sum_{y\in \sY}p(x, y)\1_{\hh_{\sY}(x)\neq y} + c}$}\\
& \leq \paren*{\sum_{y\in \sY}p(x, y)\1_{\hh_{\sY}(x)\neq y} + c}\Gamma_2\bigg[\eta(x)\Phi\paren*{h_{n+1}(x)-\max_{y\in \sY}h_{\sY}(x, y)}+(1-\eta(x))\Phi\paren*{\max_{y\in \sY}h_{\sY}(x, y)-h_{n+1}(x)}\\
&\quad-\inf_{h_{n+1}\in \sH_{n+1}}\paren*{\eta(x)\Phi\paren*{h_{n+1}(x)-\max_{y\in \sY}h_{\sY}(x, y)}+(1-\eta(x))\Phi\paren*{\max_{y\in \sY}h_{\sY}(x, y)-h_{n+1}(x)}}\bigg]\tag{By $\sH_{n+1}^{\tau}$-consistency bounds of $\Phi$ under assumption, $\tau=\max_{y\in \sY}h_{\sY}(x, y)$}\\
& =  \paren*{\sum_{y\in \sY}p(x, y)\1_{\hh_{\sY}(x)\neq y} + c} \Gamma_2\paren*{\frac{\sum_{y\in \sY}p(x, y)\ell_{h_{\sY}}(h_{n+1},x, y)-\inf_{h_{n+1}\in \sH_{n+1}}\sum_{y\in \sY}p(x, y)\ell_{h_{\sY}}(h_{n+1},x, y)}{ \sum_{y\in \sY}p(x, y)\1_{\hh_{\sY}(x)\neq y} + c}}\tag{ $\eta(x) = \frac{\sum_{y\in \sY}p(x, y)\1_{\hh_{\sY}(x)\neq y}}{\sum_{y\in \sY}p(x, y)\1_{\hh_{\sY}(x)\neq y} + c}$ and formulation \eqref{eq:ell-Phi-h}}\\
& =  \paren*{\sum_{y\in \sY}p(x, y)\1_{\hh_{\sY}(x)\neq y} + c} \Gamma_2\paren*{\frac{\sC_{\ell_{h_{\sY}}}(h_{n+1},x)-\sC^*_{\ell_{h_{\sY}}}(\sH_{n+1},x)}{ \sum_{y\in \sY}p(x, y)\1_{\hh_{\sY}(x)\neq y} + c}}\\
& \leq
\begin{cases}
\Gamma_2\paren*{\sC_{\ell_{h_{\sY}}}(h_{n+1},x)-\sC^*_{\ell_{h_{\sY}}}(\sH_{n+1},x)} & \text{when $\Gamma_2$ is linear}\\
(1+c)\Gamma_2\paren*{\frac {\sC_{\ell_{h_{\sY}}}(h_{n+1},x)-\sC^*_{\ell_{h_{\sY}}}(\sH_{n+1},x)}{c}} & \text{otherwise}
\end{cases}\\
\tag{$c\leq \sum_{y\in \sY}p(x, y)\1_{\hh_{\sY}(x)\neq y} + c\leq 1+c$ and $\Gamma_2$ is non-decreasing}\\
& = \begin{cases}
\Gamma_2\paren*{\Delta\sC_{\ell_{h_{\sY}},\sH_{n+1}}(h_{n+1},x)}  & \text{when $\Gamma_2$ is linear}\\
(1+c)\Gamma_2\paren*{\frac {\Delta\sC_{\ell_{h_{\sY}},\sH_{n+1}}(h_{n+1},x)}{c}} & \text{otherwise}
\end{cases}
\end{align*}
and 
\begin{align*}
& \inf_{h_{n+1}\in \sH_{n+1}}\sC_{\labs}(h,x)-\sC^*_{\labs}(\sH,x)\\
& = \inf_{h_{n+1}\in \sH_{n+1}}\sC_{\labs}(h,x)-\inf_{h_{\sY}\in \sH_{\sY},h_{n+1}\in \sH_{n+1}}\sC_{\labs}(h,x)\\
& = \inf_{h_{n+1}\in \sH_{n+1}} \paren*{\sum_{y\in \sY}p(x, y)\1_{ \hh_{\sY}(x)\neq y}\1_{ \hh(x)\neq n + 1} + c \1_{ \hh(x) = n + 1}}\\
&\quad -\inf_{h_{\sY}\in \sH_{\sY},h_{n+1}\in \sH_{n+1}} \paren*{\sum_{y\in \sY}p(x, y)\1_{ \hh_{\sY}(x)\neq y}\1_{ \hh(x)\neq n + 1} + c \1_{ \hh(x) = n + 1}}\\
& = \inf_{h_{n+1}\in \sH_{n+1}} \paren*{\sum_{y\in \sY}p(x, y)\1_{ \hh_{\sY}(x)\neq y}\1_{ \hh(x)\neq n + 1} + c \1_{ \hh(x) = n + 1}}\\
&\quad -\inf_{h_{n+1}\in \sH_{n+1}} \paren*{\inf_{h_{\sY}\in \sH_{\sY}}\sum_{y\in \sY}p(x, y)\1_{ \hh_{\sY}(x)\neq y}\1_{ \hh(x)\neq n + 1} + c \1_{ \hh(x) = n + 1}}\\
& = \min\curl*{\sum_{y\in \sY}p(x, y)\1_{\hh_{\sY}(x)\neq y},c}-\min\curl*{\inf_{h_{\sY}\in \sH_{\sY}}\sum_{y\in \sY}p(x, y)\1_{\hh_{\sY}(x)\neq y},c}\\
& \leq \sum_{y\in \sY}p(x, y)\1_{\hh_{\sY}(x)\neq y} -\inf_{h_{\sY}\in \sH_{\sY}}\sum_{y\in \sY}p(x, y)\1_{\hh_{\sY}(x)\neq y}\\
& = \sC_{\ell_{0-1}}(h_{\sY}, x)-\sC^*_{\ell_{0-1}}(\sH_{\sY},x)\\
& = \Delta\sC_{\ell_{0-1},\sH_{\sY}}(h_{\sY}, x)\\
& \leq \Gamma_1\paren*{\Delta\sC_{\ell,\sH_{\sY}}(h_{\sY}, x)}.
\tag{By $\sH_{\sY}$-consistency bounds of $\ell$ under assumption}
\end{align*}
Therefore, by \eqref{eq:expression-two-step}, we obtain
\begin{align*}
& \sE_{\labs}(h)-\sE^*_{\labs}\paren*{\sH_{\sY}}+\sM_{\labs}(\sH_{\sY})\\
& \leq 
\begin{cases}
\E_X\bracket*{\Gamma_2\paren*{\Delta\sC_{\ell_{h_{\sY}},\sH_{n+1}}(h_{n+1},x)}} + \E_X\bracket*{\Gamma_1\paren*{\Delta\sC_{\ell,\sH_{\sY}}(h_{\sY}, x)}} & \text{when $\Gamma_2$ is linear}\\
(1+c)\E_X\bracket*{\Gamma_2\paren*{\frac {\Delta\sC_{\ell_{h_{\sY}},\sH_{n+1}}(h_{n+1},x)}{c}}} + \E_X\bracket*{\Gamma_1\paren*{\Delta\sC_{\ell,\sH_{\sY}}(h_{\sY}, x)}} & \text{otherwise}
\end{cases}\\
& \leq 
\begin{cases}
\Gamma_2\paren*{\E_X\bracket*{\Delta\sC_{\ell_{h_{\sY}},\sH_{n+1}}(h_{n+1},x)}} + \Gamma_1\paren*{\E_X\bracket*{\Delta\sC_{\ell,\sH_{\sY}}(h_{\sY}, x)}} & \text{when $\Gamma_2$ is linear}\\
(1+c)\Gamma_2\paren*{\frac1c\E_X\bracket*{\Delta\sC_{\ell_{h_{\sY}},\sH_{n+1}}(h_{n+1},x)}} + \Gamma_1\paren*{\E_X\bracket*{\Delta\sC_{\ell,\sH_{\sY}}(h_{\sY}, x)}} & \text{otherwise}
\end{cases}
\tag{$\Gamma_1$ and $\Gamma_2$ are concave}\\
& =
\begin{cases}
\Gamma_1\paren*{\sE_{\ell}(h)-\sE_{\ell}^*(\sH_{\sY}) +\sM_{\ell}(\sH_{\sY})} + \Gamma_2\paren*{\sE_{\ell_{h_{\sY}}}(h_{n+1})-\sE_{\ell_{h_{\sY}}}^*(\sH_{n+1}) +\sM_{\ell_{h_{\sY}}}(\sH_{n+1})} & \text{when $\Gamma_2$ is linear}\\
(\Gamma_1\paren*{\sE_{\ell}(h)-\sE_{\ell}^*(\sH_{\sY}) +\sM_{\ell}(\sH_{\sY})} + (1+c)\Gamma_2\paren*{\frac{\sE_{\ell_{h_{\sY}}}(h_{n+1})-\sE_{\ell_{h_{\sY}}}^*(\sH_{n+1}) +\sM_{\ell_{h_{\sY}}}(\sH_{n+1})}{c}} & \text{otherwise},
\end{cases}
\end{align*}
which completes the proof.

\end{proof}

\subsection{Proof of Realizable \texorpdfstring{$\sH$}{H}-Consistency
  for Two-Stage Surrogates
  (Theorem~\ref{Thm:bound-general-two-step-realizable})}
\label{app:bound-general-two-step-realizable}

\begin{definition}[\textbf{Realizable $\sH$-consistency}]
\label{def:rel-consistency} Let $\hat h$ denote a hypothesis attaining the infimum of the expected surrogate loss, $\sE_{\sfL}(\hat h) = \sE^*_{\sfL}(\sH)$. A score-based abstention surrogate loss $\sfL$ is said to be
\emph{realizable $\sH$-consistent} with respect to the abstention loss
$\labs$ if, for any distribution in which an optimal hypothesis $h^*$
exists in $\sH$ with an abstention loss of zero (i.e.,
$\sE_{\labs}(h^*)=0$), we have $\sE_{\labs}(\hat h) = 0$.
\end{definition}

Next, we demonstrate that our proposed two-stage score-based surrogate
losses are not only Bayes-consistent, as previously established in
Section~\ref{sec:two-stage}, but also realizable $\sH$-consistent,
which will be shown in
Theorem~\ref{Thm:bound-general-two-step-realizable}. This effectively
addresses the open question posed by \citet{pmlr-v206-mozannar23a} in
the context of score-based multi-class abstention and highlights the
benefits of the two-stage formulation.

\begin{restatable}[\textbf{Realizable $\sH$-consistency for
      two-stage surrogates}]{theorem}{BoundGenralTwoStepRealizable}
\label{Thm:bound-general-two-step-realizable}

Given a hypothesis set $\sH=\sH_{\sY}\times \sH_{n+1}$ that is closed
under scaling.  Let $\Phi$ be a function that satisfies the condition
$\lim_{t\to \plus \infty}\Phi(t)=0$ and $\Phi(t) \geq 1_{t \leq 0}$
for any $ t \in \Rset$. Assume that $\hat h = (\hat h_{\sY}, \hat
h_{n+1})\in \sH$ attains the infimum of the expected surrogate loss,
$\sE_{\ell}(\hat h_{\sY}) = \inf_{h_{\sY} \in
  \sH_{\sY}}\sE_{\ell}(h_{\sY})$ and $\sE_{\ell_{\hat h_{\sY}}}(\hat
h_{n+1}) = \inf_{h \in \sH}\sE_{\ell_{h_{\sY}}}(h_{n+1})$. Then, for
any distribution in which an optimal hypothesis
$h^*=(h^*_{\sY},h^*_{n+1})$ exists in $\sH$ with $\sE_{\labs}(h^*)=0$,
we have $\sE_{\labs}(\hat h) = 0$.
\end{restatable}

\begin{proof}
By the assumptions, $\ell_{h_{\sY}}$ serves as an upper bound for
$\labs$ and thus $\sE_{\labs}(\hat h) \leq \sE_{\ell_{\hat
    h_{\sY}}}(\hat h_{n + 1})$. If abstention happens, that is
$h^*_{n+1}(x) > \max_{y\in \sY}h^*_{\sY}(x, y)$ for some point $x$,
then we must have $c=0$ by the realizability assumption. Therefore,
there exists an optimal $h^{**}$ such that $h^{**}_{n+1}(x) >
\max_{y\in \sY}h^{**}_{\sY}(x, y)$ for all $x \in \sX$ without
incurring any cost.  Then, by the Lebesgue dominated convergence
theorem and the assumption that $\sH$ is closed under scaling,
\begin{align*}
& \sE_{\labs}(\hat h)\\
& \leq \sE_{\ell_{\hat h_{\sY}}}(\hat h_{n + 1}) \\
& \leq \lim_{\alpha \to \plus\infty}\sE_{\ell_{\alpha h^{**}_{\sY}}}(\alpha h^{**}_{n+1})\\
& =\lim_{\alpha \to \plus\infty}\mathbb{E}\bracket*{ \ell_{\alpha h^{**}_{\sY}}\paren*{\alpha h^{**}_{n+1}, x, y}  }\\
& =\lim_{\alpha\to \plus\infty}\mathbb{E}\bracket*{ \1_{ \hh^{**}_{\sY}(x) \neq y} \Phi\paren*{\alpha\paren*{h^{**}_{n+1}(x)-\max_{y\in \sY}h^{**}_{\sY}(x, y)}} + c \Phi\paren*{\alpha\paren*{\max_{y\in \sY} h^{**}_{\sY}(x, y)-h^{**}_{n+1}(x)}}}\\
& =\lim_{\alpha\to \plus\infty}\mathbb{E}\bracket*{\1_{ \hh^{**}_{\sY}(x) \neq y} \Phi\paren*{\alpha\paren*{ h^{**}_{n+1}(x)-\max_{y\in \sY} h^{**}_{\sY}(x, y)}}} \tag{$c=0$}\\
& =0.  \tag{using $\lim_{t\to \plus \infty}\Phi(t)=0$ and the Lebesgue dominated convergence theorem}
\end{align*}
If abstention does not happen, that is $h^*_{n+1}(x)-\max_{y\in \sY}h^*_{\sY}(x, y)<0$ for all $x \in \sX$, then we must have $h^*_{\sY}(x, y)-\max_{y'\neq y}h^*_{\sY}(x, y')>0$ for all $x \in \sX$ and $y \in \sY$ by the realizability assumption. Then, by the Lebesgue dominated convergence theorem and the assumption that $\sH$ is closed under scaling,
\begin{align*}
& \sE_{\labs}(\hat h)\\
& \leq \sE_{\ell_{\hat h_{\sY}}}(\hat h_{n + 1})\\
& \leq \lim_{\alpha \to \plus\infty}\sE_{\ell_{\alpha h^*_{\sY}}}(\alpha h^{*}_{n+1})\\
& = \lim_{\alpha \to \plus\infty}\mathbb{E}\bracket*{ \ell_{ \alpha h^{*}_{\sY}}\paren*{\alpha h^{*}_{n+1}, x, y}  }\\
& = \lim_{\alpha\to \plus\infty}\mathbb{E}\bracket*{ \1_{ \hh^{*}_{\sY}(x) \neq y} \Phi\paren*{\alpha\paren*{h^*_{n+1}(x)-\max_{y\in \sY}h^*_{\sY}(x, y)}} + c \Phi\paren*{\alpha\paren*{\max_{y\in \sY} h^*_{\sY}(x, y)-h^*_{n+1}(x)}}}\\
& = \lim_{\alpha\to \plus\infty}\mathbb{E}\bracket*{ c \Phi\paren*{\alpha\paren*{\max_{y\in \sY} h^*_{\sY}(x, y)-h^*_{n+1}(x)}}} \tag{$h^{*}_{\sY}(x, y)-\max_{y'\neq y} h^{*}_{\sY}(x, y')>0$}\\
& = 0.  \tag{using $\lim_{t\to \plus \infty}\Phi(t)=0$ and the Lebesgue dominated convergence theorem}
\end{align*}
By combining the above two analysis, we conclude the proof.
\end{proof}
\ignore{
\begin{proof}
For any distribution in which an optimal hypothesis
$h^*=(h^*_{\sY},h^*_{n+1})$ exists in $\sH$ with $\sE_{\labs}(h^*)=0$,
we have for any $x\in \sX$, either $c=0$ and $h^*_{n+1}(x)> \max_{y\in
  \sY}h^*_{\sY}(x, y)$, or there exists $y_{\max}$ such that
$p(x, y_{\max})=1$, $h^*_{\sY}(x, y_{\max})> \max_{y'\neq
  y_{\max}}h^*_{\sY}(x, y')$ and $\max_{y\in
  \sY}h^*_{\sY}(x, y)>h^*_{n+1}(x)$.  Since $\sH$ is closed under
scaling, $\alpha h^*\in \sH$ for any $\alpha>0$. Using the fact that
$\lim_{t\to \plus \infty}\Phi(t)=0$ and $\lim_{\alpha \to \plus
  \infty}\sE_{\ell}(\alpha h^*_{\sY})=0$ for $\ell$ being the logistic
loss, we obtain
$\sE_{\ell}^*(\sH_{\sY})=\sE_{\ell_{h_{\sY}}}^*(\sH_{n+1})=0$. By
establishing that $\ell_{h_{\sY}}$ serves as an upper bound for
$\labs$, we conclude the proof.
\end{proof}
}

\section{Significance of \texorpdfstring{$\sH$}{H}-consistency bounds with minimizabiliy gaps}
\label{app:better-bounds}

As previously highlighted, the minimizabiliy gap can be upper bounded
by the approximation error $\sA_{\sfL}(\sH)= \sE^*_{\lsc}(\sH) -
\E_x\bracket[\big]{\inf_{ h \in \sH_{\rm{all}}} \E_y \bracket{\lsc( h,
    X, y) \mid X =
    x}}=\sE^*_{\sfL}(\sH)-\sE^*_{\sfL}(\sH_{\rm{all}})$. However, it
is a finer quantity than the approximation error, and as such, it can
potentially provide more significant guarantees.  To elaborate, as
shown by \citep{awasthi2022Hconsistency,awasthi2022multi}, for a
target loss function $\sfL_2$ and a surrogate loss function $\sfL_1$,
the excess error bound $\sE_{\sfL_2} (h) -
\sE^*_{\sfL_2}(\sH_{\rm{all}})\leq \Gamma\paren*{ \sE_{\sfL_1} (h) -
  \sE^*_{\sfL_1}(\sH_{\rm{all}})}$ can be reformulated as
\begin{align*}
\sE_{\sfL_2} (h) - \sE^*_{\sfL_2}(\sH) +\sA_{\sfL_2}(\sH)\leq \Gamma\paren*{ \sE_{\sfL_1} (h) - \sE^*_{\sfL_1}(\sH)+\sA_{\sfL_1}(\sH)},
\end{align*}
where $\Gamma$ is typically linear or the square-root function modulo
constants.  On the other hand, an $\sH$-consistency bound can be
expressed as follows:
\begin{equation*}
\sE_{\sfL_2} (h) - \sE^*_{\sfL_2}(\sH) +  \sM_{\sfL_2}(\sH)  \leq \Gamma\paren*{ \sE_{\sfL_1} (h) - \sE^*_{\sfL_1}(\sH) + \sM_{\sfL_1}(\sH}.
\end{equation*}
For a target loss function $\sfL_2$ with discrete outputs, such as the
zero-one loss or the deferral loss, we have
$\E_{x}\bracket[\big]{\inf_{h \in\sH}\E_{y}\bracket*{\sfL_2(h, x,
    y)\mid X = x}}=\E_x\bracket[\big]{\inf_{h \in \sH_{\rm{all}}} \E_{y}
  \bracket*{\sfL_2(h,x, y)\mid X = x}}$ when the hypothesis set generates labels
that cover all possible outcomes for each input (See
\citep[Lemma~3]{awasthi2022multi},
Lemma~\ref{lemma:calibration_gap_score} in
Appendix~\ref{app:score}). Consequently, we have
$\sM_{\sfL_2}(\sH) = \sA_{\sfL_2}(\sH)$. However, for a surrogate loss function
$\sfL_1$, the minimizability gap is upper bounded by the approximation
error, $\sM_{\sfL_1}(\sH)\leq \sA_{\sfL_1}(\sH)$, and is generally
finer.

Let us consider a straightforward binary classification example where the conditional distribution is denoted as $\eta(x)=D(Y=1 | X = x)$. We will define $\sH$ as a set of functions $h$, such that $|h(x)| \leq \Lambda$ for all $x \in \sX$, for some $\Lambda > 0$, and it is also possible to achieve any value in the range $[-\Lambda, +\Lambda]$. For the exponential-based margin loss, which we define as $\sfL(h, x, y) = e^{-yh(x)}$, we obtain the following equation:
\begin{equation*}
\E_{y}[\sfL(h, x, y)\mid X = x] = \eta(x)
e^{-h(x)} + (1 - \eta(x)) e^{h(x)}.
\end{equation*}
Upon observing this, it becomes apparent that the infimum over all measurable functions can be expressed in the following way, for all $x$:
\begin{equation*}
\inf_{h
  \in \sH_{\mathrm{all}}}\E_{y}[\sfL(h, x, y)\mid X = x] =
2\sqrt{\eta(x)(1-\eta(x))},
\end{equation*}
while the infimum over $\sH$, $\inf_{h \in \sH}\E_{y}[\sfL(h, x, y)\mid X = x]$, depends on $\Lambda$ and can be expressed as
  \begin{equation*}
   \inf_{h \in \sH}\E_{y}[\sfL(h, x, y)\mid X = x]=\begin{cases}
   \max\curl*{\eta(x),1 - \eta(x)}
e^{-\Lambda} + \min\curl*{\eta(x),1 - \eta(x)} e^{\Lambda} & \Lambda<\frac{1}{2} \abs*{\log \frac{\eta(x)}{1
    -\eta(x)}}\\
    2\sqrt{\eta(x)(1-\eta(x))} & \text{otherwise}.
   \end{cases} 
  \end{equation*}
Thus, in the deterministic scenario, the discrepancy between the
approximation error $\sA_{\sfL}(\sH)$ and the minimizability gap
$\sM_{\sfL}(\sH)$ is:
\begin{equation*}
  \sA_{\sfL}(\sH)-\sM_{\sfL}(\sH)
  = \E_{x}\bracket*{\inf_{h \in\sH}\E_{y}\bracket*{\sfL(h, x, y)\mid X = x}
    - \inf_{h \in
      \sH_{\rm{all}}} \E_{y} \bracket*{\sfL(h,x, y)\mid X = x}}
  = e^{-\Lambda}.
\end{equation*}
Therefore, for a surrogate loss, the minimizability gap can be
strictly less than the approximation error.  In summary, an
$\sH$-consistency bound can be more significant than the excess error
bound as $\sM_{\sfL_2}(\sH) = \sA_{\sfL_2}(\sH)$ when $\sfL_2$
represents the zero-one loss or deferral loss, and $\sM_{\sfL_1}(\sH)
\leq \sA_{\sfL_1}(\sH)$. They can also be directly used to derive
finite sample estimation bounds for a surrogate loss minimizer, which
are more favorable and relevant than a similar finite sample guarantee
that could be derived from an excess error bound (see
Section~\ref{sec:finite-sample}).

\vfill

\end{document}